\documentclass[sigconf, nonacm]{acmart}

\usepackage[T1]{fontenc}
\usepackage[utf8]{inputenc}

\usepackage{bbm, dsfont}
\usepackage[most]{tcolorbox}
\usepackage{fontawesome5}

\usepackage{makecell}
\usepackage{tabularx}
\usepackage{booktabs}
\usepackage{listings}     
\usepackage{tcolorbox}
\usepackage{upquote}
\usepackage{xcolor}
\usepackage{booktabs}
\usepackage{multirow}
\usepackage{mathtools}
\usepackage{wrapfig}
\usepackage{nicefrac}
\usepackage{amsmath}
\usepackage{soul}
\usepackage{graphicx}  
\usepackage{subcaption}  
\usepackage{caption}  
\usepackage[linesnumbered,ruled,vlined,noline]{algorithm2e}
\usepackage{enumitem} 
\usepackage{amsthm}
\newtheorem{hypothesis}{Hypothesis}
\newtheorem{problem}{Problem}

\usepackage{colortbl} 
\usepackage{xcolor} 

\usepackage{algpseudocode} 

\SetCommentSty{mycommentstyle}

\algnewcommand\algorithmicinput{\textbf{Input:}}
\algnewcommand\algorithmicoutput{\textbf{Output:}}
\algnewcommand\Input{\item[\algorithmicinput]}
\algnewcommand\Output{\item[\algorithmicoutput]}
\algnewcommand\AlgoComment[1]{\hfill\(\triangleright\) \textcolor{gray}{\textit{#1}}}

\tcbuselibrary{listings}  
\tcbuselibrary{skins}     
\tcbuselibrary{listings,breakable}

\definecolor{dkgreen}{rgb}{0,0.6,0}
\definecolor{gray}{rgb}{0.5,0.5,0.5}
\definecolor{mauve}{rgb}{0.58,0,0.82}

\newcommand{\add}[1]{\textcolor{blue}{#1}}
\newcommand{\marked}[1]{\textcolor{red}{#1}}

\newcommand{\yuyu}[1]{\textcolor{blue}{[\textbf{Yuyu: }#1]}}

\newcommand{\lxt}[1]{\textcolor{orange}{[\textbf{Xiaotian: }#1]}}

\newcommand{\stab}{\vspace{1.2ex}\noindent}
\newcommand{\stitle}[1]{\stab\noindent{\textbf{#1}}}
\newcommand{\etitle}[1]{\vspace{1mm}\noindent{\underline{\em #1}}}
\newcommand{\vs}{\textit{vs.}\xspace}
\newcommand{\ie}{\textit{i.e.,}\xspace}
\newcommand{\eg}{\textit{e.g.,}\xspace}

\newcommand{\model}{LEAD\xspace}
\newcommand{\lead}{LEAD\xspace}

\NewDocumentCommand{\nan}{ mO{} }{\textcolor{blue}{\textsuperscript{\textit{Nan}}\textsf{\textbf{\small[#1]}}}}

\newtheoremstyle{mystyle}       
  {6pt}                        
  {6pt}                        
  {\normalfont}                   
  {0pt}                         
  {\bfseries}                   
  {.}                           
  {5pt plus 1pt minus 1pt}      
  {\thmname{\textsc{#1}}\thmnumber{ #2}}  

\theoremstyle{mystyle}

\tcbset{
  stepbox/.style={
    enhanced,
    on line,               
    valign=center,         
    colback=black,
    coltext=white,
    arc=2pt,
    boxrule=0pt,
    boxsep=0pt,
    left=1pt,
    right=1pt,
    top=1.25pt,
    bottom=1pt,
    fontupper=\sffamily\bfseries\footnotesize,
  }
}

\newcommand{\step}[1]{\tcbox[stepbox]{Step~#1}}

\setlist[itemize]{leftmargin=0.4cm, topsep=2pt, itemsep=1pt}  

\newcommand\vldbdoi{XX.XX/XXX.XX}

\newcommand\vldbvolume{14}
\newcommand\vldbissue{1}

\newcommand\vldbavailabilityurl{URL_TO_YOUR_ARTIFACTS}

\begin{document}

\title{\mbox{LEAD: Iterative Data Selection for Efficient LLM Instruction Tuning}}

\author{Xiaotian Lin}
\affiliation{%
	\institution{HKUST (GZ)}
	\country{Guangzhou, China}
}

\author{Yanlin Qi}
\affiliation{%
	\institution{Université Paris Cité}
	\country{Paris, France}
}

\author{Yizhang Zhu}
\affiliation{%
	\institution{HKUST (GZ)}
	\country{Guangzhou, China}
}

\author{Themis Palpanas}
\affiliation{%
	\institution{Université Paris Cité}
	\country{Paris, France}
}

\author{Chengliang Chai}
\affiliation{%
	\institution{BIT}
	\country{Beijing, China}
}

\author{Nan Tang}
\affiliation{%
	\institution{HKUST (GZ)}
	\country{Guangzhou, China}
}

\author{Yuyu Luo$^*$}
\affiliation{%
	\institution{HKUST (GZ)}
	\country{Guangzhou, China}
}

\begin{abstract}
Instruction tuning has emerged as a critical paradigm for improving the capabilities and alignment of large language models (LLMs). However, existing iterative model-aware data selection methods incur significant computational overhead, as they rely on repeatedly performing full-dataset model inference to estimate sample utility for subsequent training iterations, creating a fundamental efficiency bottleneck.
In this paper, we propose \model, an efficient iterative data selection framework that accurately estimates sample utility entirely within the standard training loop, eliminating the need for costly additional model inference. At its core, \model introduces Instance-Level Dynamic Uncertainty (IDU), a theoretically grounded utility function combining instantaneous training loss, gradient-based approximation of loss changes, and exponential smoothing of historical loss signals. To further scale efficiently to large datasets, \model employs a two-stage, coarse-to-fine selection strategy, adaptively prioritizing informative clusters through a multi-armed bandit mechanism, followed by precise fine-grained selection of high-utility samples using IDU.
Extensive experiments across four diverse benchmarks show that \model significantly outperforms state-of-the-art methods, improving average model performance by 6.1\%-10.8\% while using only 2.5\% of the training data and reducing overall training time by 5-10×.
\end{abstract}

\maketitle



\begingroup
\renewcommand\thefootnote{}\footnote{\noindent
$^*$Corresponding author: Yuyu Luo (E-mail: yuyuluo@hkust-gz.edu.cn)

\noindent This work is licensed under the Creative Commons BY-NC-ND 4.0 International License. Visit \url{https://creativecommons.org/licenses/by-nc-nd/4.0/} to view a copy of this license. For any use beyond those covered by this license, obtain permission by emailing \href{mailto:info@vldb.org}{info@vldb.org}. Copyright is held by the owner/author(s). Publication rights licensed to the VLDB Endowment. \\
\raggedright Proceedings of the VLDB Endowment, Vol. \vldbvolume, No. \vldbissue\ %
ISSN 2150-8097. \\
\href{https://doi.org/\vldbdoi}{doi:\vldbdoi} \\
}\addtocounter{footnote}{-1}\endgroup


\ifdefempty{\vldbavailabilityurl}{}{
\vspace{.3cm}
\begingroup\small\noindent\raggedright\textbf{PVLDB Artifact Availability:}\\
The source code, data, and/or other artifacts have been made available at \url{https://github.com/HKUSTDial/LEAD}.
\endgroup
}


\section{Introduction}
\label{sec:intro}

\begin{figure}[t!]
  \includegraphics[width=\columnwidth]{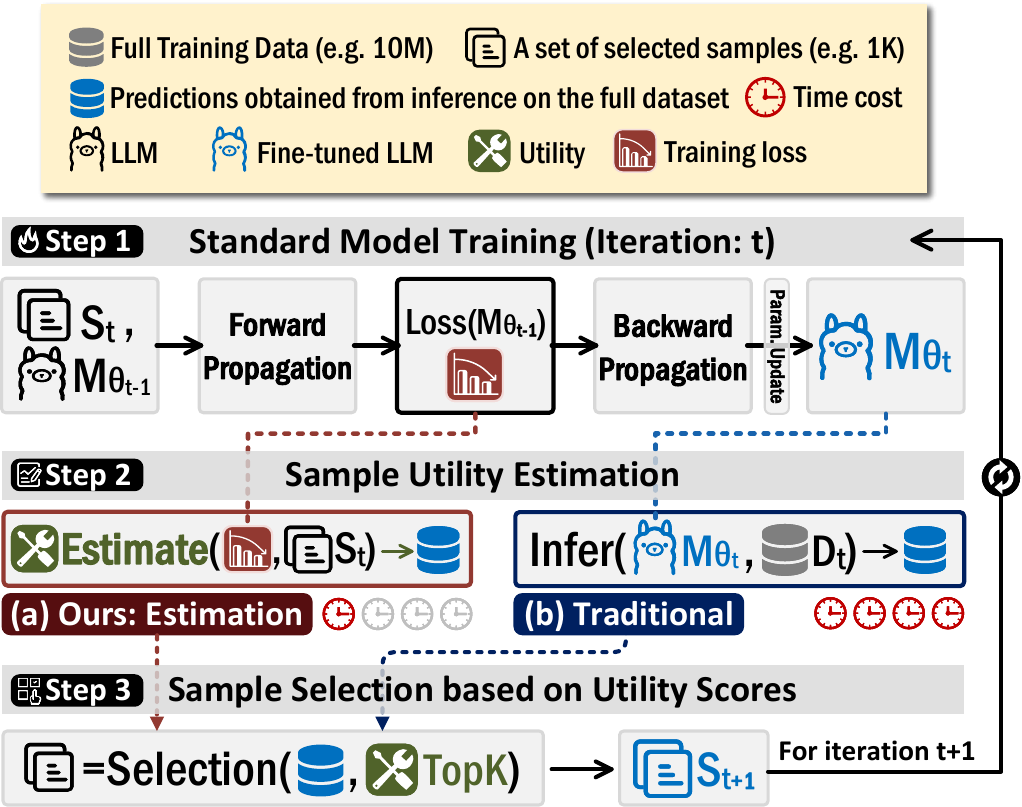} 
  \vspace{-2em}
  \caption{Comparison of Iterative Model-Aware Solutions.}
  \label{fig:intro}
  \vspace{-1em}
\end{figure}

Instruction tuning has emerged as a powerful paradigm to improve the performance and alignment of large language models (LLMs) by fine-tuning them on instruction-response pairs~\cite{sun2024itd, li2024quantity, chang2024survey, albalaksurvey,DBLP:conf/cidr/0001YF0LH24}. 
Recent studies indicate that data quality, rather than quantity alone, is crucial for substantial performance gains~\cite{zhou2023lima, albalaksurvey,DBLP:journals/pvldb/LiLCLT24,liu2025surveynl2sqllargelanguage}. Consequently, recent research has focused on automatically selecting informative subsets of training data, guided by selection metrics such as data diversity and data quality~\cite{wu2023self, yu2024diversify, bukharin2024data,DBLP:journals/vldb/QinLTL20,DBLP:journals/tkde/ChaiWLNL23}. However, since these methods do not directly leverage feedback from the model, they fail to dynamically adapt data selection to the model's evolving state and specific learning needs throughout training.

In response, recent efforts have shifted toward \textit{model-aware data selection}, which explicitly utilizes model-derived signals to dynamically identify informative training examples~\cite{song2024iterselecttune, wang2024greats}.
These model-aware methods broadly fall into two categories: \textit{non-iterative} and \textit{iterative}. Non-iterative methods select data once based on initial model predictions before iterative training~\cite{xia2024less, li2024quantity}. However, since they do not adapt to the model evolvement during training, their effectiveness is inherently limited~\cite{xia2024rethinking}. In contrast, iterative methods interleave model fine-tuning and data selection across multiple rounds, iteratively choosing new informative samples based on the model's latest feedback~\cite{xia2024less}.
As shown in Figure~\ref{fig:intro}-\step{2-(b)}, most existing iterative model-aware methods typically rely on explicit model inference to assess the utility of samples. Specifically, after each training iteration, these methods perform inference on \textit{every} sample in the training set to derive feedback signals (\eg model uncertainty scores) for utility estimation. Although effective at adapting data selection to the model's evolving state, repeatedly performing full-dataset inference significantly increases computational overhead.
For example, the recent IFD method~\cite{li2024quantity} spends approximately $98$ GPU-hours selecting data from a pool of only $600K$ samples in a single round.

This predicament leads to a natural \textbf{research question}: 
\textit{Can we retain the benefits of iterative model-aware data selection without repeatedly performing costly full-dataset inference?}
In other words, can we effectively determine ``select what to learn next'' by exclusively utilizing information already computed during standard training, without any additional model inference overhead?

In this work, we posit that the answer is yes. As shown in Figure~\ref{fig:intro}-\step{1}, our key insight is that during standard training, the model first conducts a forward propagation step using the current mini-batch of samples, computes the per-sample losses based on its predictions, and subsequently updates its parameters via backward propagation. Crucially, this training process naturally produces a per-sample loss for each training instance in the mini-batch. Intuitively, this loss indicates how challenging a sample is for the model—higher losses reflect greater difficulty and thus greater potential informativeness for future learning. Hence, these training-time losses inherently serve as valuable indicators of a sample's utility. Indeed, they provide an effective proxy for explicit utility metrics (\eg model uncertainty) typically obtained through costly, separate inference steps~\cite{han2025automatic}.

\textit{If we can cleverly harness these inherent training signals across the whole dataset}, we could \textbf{estimate} the utility of each sample \textbf{without additional inference (inference-free)} (see Figure~\ref{fig:intro}-\step{2-(a)}).
This idea – leveraging training-time loss signals to guide data selection – offers the potential to eliminate the full-dataset inference stage while still adapting to the model's training state. 


\stitle{Challenges.} 
Realizing this idea in practice is non-trivial.

First, although using training-time losses allows us to avoid explicit inference, a subtle yet fundamental issue arises due to a timing misalignment. Specifically, as shown in Figure~\ref{fig:intro}-\step{1}, the training loss observed at iteration $t$ reflects the model's performance \textit{before} updating parameters (model state $M_{\theta_{t-1}}$), whereas the utility of selecting samples ideally should consider their usefulness \textit{after} the parameter update (\ie $M_{\theta_{t}}$ at iteration $t+1$). This temporal mismatch means that naively reusing pre-update loss signals may not accurately reflect true sample utility after the next parameter update. We term this issue as the \textit{temporal mismatch} challenge (\textbf{C1}).

Second, raw loss signals can be noisy or unstable – they fluctuate from one update to the next due to randomness (\eg varying batch composition) and the non-stationary nature of training, thus naively trusting instantaneous loss values might lead to suboptimal choices.
This issue highlights the \textit{instability of loss signals} challenge (\textbf{C2}). 
Third, even if we successfully eliminate separate inference steps, individually estimating utility and selecting informative samples remains inefficient for large-scale datasets (\eg containing millions of samples). We refer to this as the  \textit{sample-level selection efficiency challenge} (\textbf{C3}). Thus, we need an effective mechanism that can rapidly narrow down candidate samples while prioritizing those most likely to substantially improve the model.

\stitle{Our Methodology: Iterative Data Selection with Inference-free Utility Estimation.}
To address the above challenges, we propose \model, a theoretically-grounded iterative data selection framework that integrates seamlessly into the model training loop, accurately estimating sample utility without incurring additional inference overhead.
The core theoretical insight behind inference-free yet accurate utility estimation lies in effectively addressing two critical challenges: (\textbf{C1}) the temporal mismatch between loss computation and parameter updates, and (\textbf{C2}) the inherent instability of instantaneous loss signals.

To achieve this, we propose a novel sample utility estimation function called {\em Instance-Level Dynamic Uncertainty} (\textbf{IDU}).
IDU explicitly implements the \texttt{Estimate} step depicted in Figure~\ref{fig:intro}-\step{2-(a)} by combining three naturally available training signals: (1) the current training loss for each sample, (2) gradient-based approximation, derived from gradient correlation approximations, to anticipate loss changes at the next parameter update (addressing \textbf{C1}), and (3) historical loss trends via exponential smoothing to reduce random noise and improve stability (addressing \textbf{C2}). Importantly, IDU is computed entirely using training-time signals naturally available during model updates (losses and logits), thus incurring no additional inference overhead. Finally, we conduct a Lagrangian function and utilize complementary slackness conditions to rigorously derive optimal parameters for IDU, ensuring both theoretical soundness and practical effectiveness.

Guided by this theoretical foundation, our LEAD framework employs a practical coarse-to-fine data selection strategy (Figure~\ref{fig:method}).

\begin{figure}[t!]
  \includegraphics[width=\columnwidth]{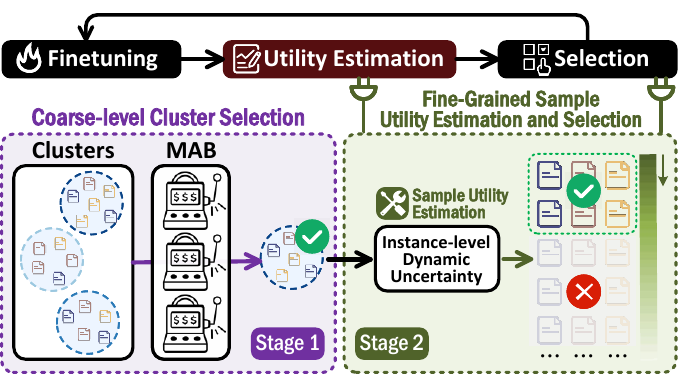}
  \vspace{-2em}
  \caption{A High-level Overview of \model.}
  \vspace{-2em}
  \label{fig:method}
\end{figure}

\paragraph{Stage 1: Coarse-level Cluster Selection.}
Recall our third challenge (\textbf{C3}) -- efficient candidate selection at scale. To address this, we first partition the dataset offline into clusters based on two widely-used metrics: (1) \textit{instruction-following difficulty}, measuring how challenging each instruction is for the model~\cite{li2024quantity}, and (2) \textit{task-level similarity}, grouping semantically related instructions~\cite{li2024instruction}.
This clustering step is performed only once per dataset.
During training, \model employs a multi-armed bandit (MAB) algorithm~\cite{vermorel2005multi} to dynamically identify and prioritize clusters likely to yield higher rewards -- clusters containing samples with greater potential to significantly enhance the model's performance (addressing \textbf{C3}).

\paragraph{Stage 2: Fine-Grained Sample Utility Estimation and Selection.} Within each selected cluster, \model utilizes the IDU function to estimate the utility of individual samples precisely. Specifically, given the IDU scores computed based on the previously discussed training signals (losses, historical trends, and gradient predictions), \model prioritizes and selects samples with the highest IDU values. Therefore, samples predicted to yield higher improvements for the model after subsequent parameter updates are selected preferentially.

\stitle{Contributions.} This paper makes the following contributions:

\stab \textit{(1) Problem Formulation.} We formally introduce the problem of Iterative Data Selection with Inference-Free Utility Estimation, defining a scenario where iterative model-aware selection is performed without incurring additional inference overhead (Section~\ref{sec:problem}).

\stab \textit{(2) Instance-Level Dynamic Uncertainty (IDU).}
We develop a new sample utility estimation function, IDU, which effectively addresses temporal mismatch and instability in loss signals by integrating current losses, gradient-based approximations of loss changes, and exponential smoothing of historical loss signals. All components are computed directly from naturally available training signals without requiring additional model inference (Section~\ref{sec:pre_idu}).

\stab \textit{(3) LEAD Framework.} 
We propose \model, a theoretically grounded and efficient iterative data selection framework seamlessly integrated into the standard model training process, eliminating repeated costly inference steps (Section~\ref{sec:overview} and Section~\ref{sec:details}).

\stab \textit{(4) Theoretical Analysis.} 
We rigorously ground our framework in a Lagrangian optimization formulation, employing complementary slackness conditions and gradient correlation approximations to derive theoretically optimal parameters for the IDU function, ensuring both soundness and practical effectiveness (Section~\ref{sec:proof}).

\stab \textit{(5) Extensive Experiments.}
Extensive experiments across four diverse benchmarks show that \model significantly outperforms state-of-the-art methods, improving average model performance by 6.1\%-10.8\% while using only 2.5\% of the training data and reducing overall training time by 5-10× (Section~\ref{sec:expr}).

\section{Preliminary and Problem Formulation}
\label{sec:problem}

\subsection{Instruction Tuning for LLMs}

%
Instruction tuning fine-tunes pretrained large language models using instruction-response pairs, enabling them to generalize to new tasks by interpreting diverse instructions~\cite{wang2023far}.
Formally, given instruction-response pairs $(x, y)$ from dataset $\mathcal{D}$, instruction tuning optimizes model $\theta$ by minimizing the expected loss:

\begin{equation}\small
\min_{\theta} \mathbb{E}_{(x, y) \sim \mathcal{D}} \left[{L}(\mathcal{M}_\theta(x), y)\right]
\end{equation}

\noindent
where ${L}$ is a task-specific loss function such as cross-entropy.

\subsection{Data Selection for Instruction Tuning}


In practice, datasets often originate from vast and noisy sources. Given limited computational budgets and data quality concerns, selecting the most informative samples for instruction tuning becomes crucial. We formalize this as the {data selection problem}, categorized into two groups: static and iterative data selection.

\stitle{Static Data Selection for Instruction Tuning.} 
Given a dataset $\mathcal{D}$, it selects a fixed subset $\mathcal{D}^* \subseteq \mathcal{D}$ under budget constraint $B$: 

\begin{equation}\small
\label{eq:static_selection}
\min_{\mathcal{D}^*\subseteq \mathcal{D},\ |\mathcal{D}^*|\leq B} \mathbb{E}_{(x,y)\sim \mathcal{D}_{\text{target}}}\left[{L}(\mathcal{M}_{\theta}(x),y)\right],
\end{equation}

\noindent
where $\mathcal{D}_{\text{target}}$ denotes the target distribution. However, static methods cannot adaptively select samples based on the model's evolving capabilities to maximize learning effectiveness during training~\cite{albalaksurvey}.

\stitle{Iterative Data Selection for Instruction Tuning.}
Iterative data selection interleaves model fine-tuning and data selection across multiple iterations. Formally, given the model parameters $\theta_t$ at iteration $t$, we adaptively select a subset $S_t \subseteq \mathcal{D}$ based on a utility function $f(\theta_t, x)$, which estimates the expected contribution of each sample $x$ to future model improvement (\eg loss reduction). The iterative selection problem can thus be formulated as:
\begin{equation}\small
\label{eq:iterative_selection_traditional}
\begin{aligned}
&\max_{\{S_1,\dots,S_T\}}\sum_{t=1}^{T}\sum_{x\in S_t} f_t(\theta_t, x),
&\text{s.t.}\quad \sum_{t=1}^{T}|S_t|\leq B,
\end{aligned}
\end{equation}
where $B$ is the total sample selection budget allowed during training.

Existing methods typically estimate the utility $f_t(\theta_t,x)$ by performing full-dataset inference at each iteration. Specifically, after fine-tuning the model on selected samples $S_t$, traditional methods explicitly run inference on the entire dataset $\mathcal{D}$ using the updated model parameters $\theta_t$ to compute utility scores:
\begin{equation}
f_t(\theta_t, x) = g(\text{Infer}(\theta_t, x)), \quad\forall x\in \mathcal{D},
\end{equation}
where $\text{Infer}(\theta_t, x)$ denotes inference (\eg loss or uncertainty computation) and $g(\cdot)$ maps inference results to utility values. 

Consequently, the next subset $S_{t+1}$ is selected as:
\begin{equation}\small
S_{t+1} = \underset{S_t\subseteq \mathcal{D},\;|S_t|\le k}{\arg\max}\sum_{x\in \mathcal{D}}f_t(\theta_t,x),\quad\text{s.t.} \quad |S_t|\leq k, \quad T\cdot k\leq B.
\end{equation}

Note that in iterative data selection, we typically assume a fixed selection size $k$ per iteration, constrained by the total selection budget $B$. Thus, the number of iterations $T$ and the selection size per iteration $k$ satisfy the relation $T\cdot k\leq B$.


\subsection{Problem Formulation}
\label{sub:problem}

Existing iterative model-aware methods rely heavily on repeated full-dataset inference for sample utility estimation, leading to significant computational costs.
To eliminate this, we define the problem of \textit{Iterative Data Selection with Inference-Free Utility Estimation}.

\begin{definition}[\textbf{Iterative Data Selection with Inference-Free Utility Estimation}]\textit{
Given a total sample selection budget $B$, our objective is to identify subsets $\{S_t\}^T_{t=1}$ that maximize the cumulative estimated utility, where the utility function $f_t(\theta_{t-1}, x)$ is computed exclusively from training-time signals (\eg training losses or logits) without incurring additional inference overhead:}

\begin{equation}\small
\label{eq:zero_cost_utility_definition}
\begin{aligned}
&\max_{\{S_1,\dots,S_T\}}\sum_{t=1}^{T}\sum_{x\in S_t}f_t(\theta_{t-1}, x), 
&\text{s.t.}\quad \sum_{t=1}^{T}|S_t|\leq B,
\end{aligned}
\end{equation}

\textit{Specifically, at each iteration $t$,  the utility estimation $f_t(\theta_{t-1},x)$ utilizes the loss signal computed using model parameters $\theta_{t-1}$ immediately after the forward propagation step, but before the backward propagation (parameter update). Thus, no additional inference is required to estimate utilities for data selection at iteration $t$.}
\end{definition}

\begin{figure*}[t!]
  \centering
  \includegraphics[width=\linewidth]{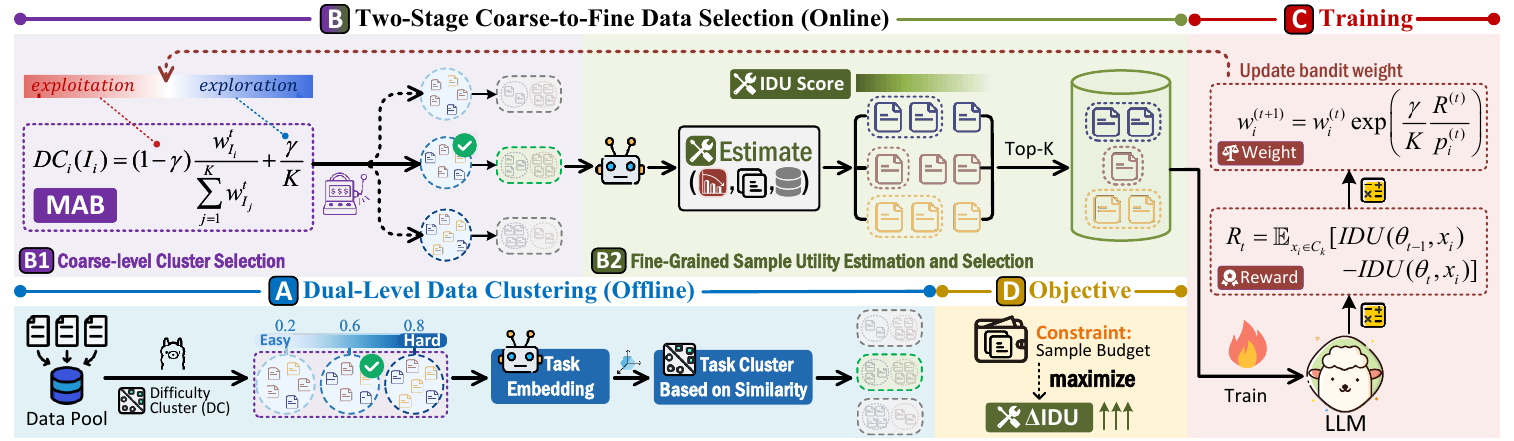}
  \vspace{-2em}
  \caption{An Overview of the \model Framework.}
  \label{fig:framework}
\end{figure*}

Our goal, therefore, is to design accurate and stable inference-free utility estimation methods. 
For simplicity, we use $f_t(\theta_{t-1},x)$ and $f(\theta_{t-1},x)$ interchangeably when the context clearly refers to data selection at iteration $t$.

\section{Instance-Level Dynamic Uncertainty Utility}
\label{sec:pre_idu}

Designing an effective inference-free utility function $f(\theta_{t-1}, x)$ requires addressing two fundamental challenges as discussed in Section~\ref{sec:intro}: 
\textbf{(C1)} the temporal mismatch between pre-update loss signals and their actual post-update utility, and (\textbf{C2}) the instability of instantaneous loss signals due to random fluctuations and noise.

To tackle these challenges, we first define a baseline utility function based on a \textit{loss-based uncertainty metric}, and then introduce an improved formulation, termed \textit{Instance-Level Dynamic Uncertainty (IDU)} utility function, which explicitly addresses these limitations.

\stitle{Loss-based Uncertainty Estimation.} Specifically, our approach begins by formalizing Instance-level uncertainty through a loss-based formulation. Formally, given an instruction-response pair $(x,y)$, we define the Instance-level Uncertainty (IU)~\cite{gardner1985exponential} at training iteration $t$ as the empirical cross-entropy between the model's current predictive distribution and the ground-truth response:

\vspace{-1em}
\begin{equation}
\small
IU(\theta_t, y|x) = L(\theta_t, x) = -\frac{1}{T}\sum_{j=1}^{T}\log p_{\theta_t}(t_j^{y}|x,t_1^{y},\dots,t_{j-1}^{y}),
\end{equation}

\noindent
where $T$ is response length, $t_{j}^{y}$ refers to the $j$-th response token, and $p_{\theta_t}$ the model's token-level predictive probability distribution.

IU naturally corresponds to the training-time negative log-likelihood loss, providing a direct and computationally free baseline. However, IU alone cannot effectively handle challenges \textbf{(C1)} and \textbf{(C2)}.

\stitle{Instance-Level Dynamic Uncertainty (IDU).} 
To explicitly mitigate both temporal mismatch (\textbf{C1}) and instability (\textbf{C2}) of loss signals, we introduce the Instance-Level Dynamic Uncertainty (IDU), which incorporates exponential smoothing of historical losses and gradient-based approximation of loss changes. 
Formally, given subset $S_t$ at iteration $t$, IDU for sample $x$ is recursively defined as:

\begin{equation} 
\label{def:utility function}
\small
\begin{split}
f(\theta_{t-1}, x) &= IDU(\theta_{t-1}, x) \\&=
(1 - b) \cdot \underbrace{[\underbrace{L(\theta_{t-1}, x)}_{\text{\textbf{IU} at $\theta_{t-1}$}} + \underbrace{\Delta L'(\theta_t, x)}_{\text{Utility Change}}]}_{\text{\textbf{Estimated} Utility at $\theta_t$}} + b \cdot \underbrace{IDU(\theta_{t-2}, x)}_{\text{Historical Utility}}
\end{split} \quad,
\end{equation}

\noindent where $b \in [0, 1)$ controls the balance between current and historical signals, $L(\theta_{t-1}, x)$ is the IU computed using model parameters $\theta_{t-1}$, and $\Delta L'(\theta_t, x)$ is an approximation of the expected utility change, defined as:
$\label{def:utility_estimation} 
\Delta L'(\theta_t, x) = L(\theta_t, x) - L(\theta_{t-1}, x)$.

We have the following key clarifications regarding Eq.~\eqref{def:utility function}:

\begin{itemize} 
    \item The instantaneous loss $L(\theta_{t-1}, x)$ is computed naturally during forward propagation at iteration $t$, requiring no extra inference.
    \item The $\Delta L'(\theta_t, x)$ denotes the anticipated loss change from $\theta_{t-1}$ to $\theta_t$. Importantly, this estimation leverages only readily available gradient and historical loss information collected at iteration $t-1$, ensuring no extra inference is performed at iteration $t$.
\end{itemize}

IDU effectively resolves both fundamental challenges through two carefully designed components:

\begin{itemize}
\item \textbf{Utility Change Estimation (Gradient-Based approximation).}
To address temporal mismatch (\textbf{C1}), IDU explicitly estimates the expected utility change ($\Delta L'(\theta_t, x)$) between consecutive iterations. Instead of performing additional inference passes with updated parameters ($\theta_t$), we leverage gradient-based approximations derived from backward propagation at iteration $t-1$ to estimate the loss at iteration $t$. 

\item  \textbf{Historical Utility (Exponential Smoothing).}
To tackle instability (\textbf{C2}), IDU incorporates historical uncertainty signals using an exponential smoothing mechanism. Rather than depending solely on instantaneous IU values, IDU maintains an exponential moving average of previous utility estimates ($IDU(\theta_{t-2}, x)$). This significantly reduces fluctuations caused by random noise and local minima encountered during training.

\end{itemize}

We will elaborate on the details of computing IDU and optimizing the coefficient $b$ of the IDU utility function in Section~\ref{sec:Instance-Level Dynamic Uncertainty}.

\section{\lead: \underline{LEA}rning to Iteratively Select \underline{D}ata}
\label{sec:overview}

We first present an overview of \model (Section~\ref{sub:overview}), followed by the three key components enabling inference-free iterative data selection (Section~\ref{sub:componets}). Finally, we describe how these components systematically interact during iterative training (Section~\ref{sub:workflow}).

\subsection{\model Framework: An Overview}
\label{sub:overview}

Figure~\ref{fig:framework} provides a high-level overview of \model, illustrating its coarse-to-fine approach guided by a theoretically grounded IDU utility function. The framework comprises two key phases: offline dual-level clustering and online adaptive selection.

\stitle{Dual-Level Data Clustering (Offline).}  
As shown in Figure~\ref{fig:framework}-(A), we first perform an offline preprocessing step to systematically partition the dataset into clusters based on two complementary dimensions: instruction-following difficulty~\cite{li2024quantity} and task similarity~\cite{li2024instruction}. This dual-level clustering is conducted offline, incurring no additional computational overhead during online training.

\etitle{(1) Difficulty-aware Instance-level Clustering.}  
We use the Instruction-Following Difficulty (IFD) metric~\cite{li2024quantity} to evaluate instance-level difficulty. Given an instruction-response pair $(x,y)$, the IFD is computed as:
$IFD(y \mid x) = \frac{PPL(y \mid x)}{PPL(y)}$, where $PPL(y \mid x)$ and $PPL(y)$ denote the perplexities of generating the $y$ with and without the $x$, respectively. Using these IFD scores, we group training samples into clusters through sliding intervals (\eg intervals of 0.1).

\etitle{(2) Similarity-based Task-level Clustering.}  
Within each difficulty cluster, we further conduct finer-grained clustering based on task similarity. Specifically, we extract task-specific embeddings from instructions by emphasizing task-defining terms (\eg key verbs and nouns), following the approach in~\cite{li2024instruction}. We then apply the $K$-means algorithm~\cite{DBLP:journals/pacmmod/LuoZ00CS23} to group instructions by task similarity.

\stitle{Coarse-to-Fine Data Selection (Online).}
During the training, as shown in Figure~\ref{fig:framework}-(B), \lead~implements a coarse-to-fine selection process designed to maximize utility and training effectiveness under a given total sample budget.

\etitle{(1) Coarse-Level Cluster Selection (via MAB).} 
    At each training iteration \( t \), we first employ a Multi-Armed Bandit (MAB) algorithm (specifically EXP3, detailed in Section~\ref{sub:mab}) to dynamically select one difficulty-level cluster that is most beneficial to the current model state. The MAB algorithm leverages a self-guided IDU-based reward signal, directly measuring the reduction in IDU scores derived from training on previously selected clusters.

\etitle{(2) Fine-Grained Sample Selection (via IDU).}
    After identifying the optimal difficulty-level cluster, we distribute the selection budget across its finer-grained task clusters. Specifically, we select the most informative samples from each task cluster based on their current IDU values (see Section~\ref{sec:Instance-Level Dynamic Uncertainty}), thus ensuring efficient fine-grained selection of training data at iteration \( t \).

These selected samples form the subset \( S_t \) used to fine-tune the model at iteration \( t \). After training, the model parameters are updated from \(\theta_{t-1}\) to \(\theta_t\), and the MAB rewards are updated accordingly, ensuring the \model framework continuously improves its data selection strategy.

\subsection{\model Framework: Core Components}
\label{sub:componets}

\model has three carefully designed core components.

\stitle{(1) Instance-Level Dynamic Uncertainty (IDU) Utility.}  
To estimate sample utility efficiently without additional inference, we introduce the Instance-Level Dynamic Uncertainty (IDU) metric. IDU combines exponential smoothing of historical losses and a gradient-based approximation of loss change, effectively addressing the temporal instability and inference overhead challenges inherent in traditional iterative selection methods (see  Section~\ref{sec:Instance-Level Dynamic Uncertainty}).

\stitle{(2) Adaptive Data Selection via MAB-Integrated Training Scheduler.}  
To integrate coarse and fine-grained selections seamlessly, we employ the MAB-EXP3 algorithm to dynamically balance exploration and exploitation among clusters. The MAB scheduler dynamically prioritizes clusters demonstrating higher historical utility gains, thus efficiently adapting to the model's evolving learning capabilities (further described in Section~\ref{sub:mab}).

\stitle{(3) Self-Guided IDU-Based Reward.}  
To guide the coarse-level cluster selection via MAB, we propose a novel reward function based on the reduction of IDU achieved by training on a given cluster without the need for external validation steps and additional inference (Please refer to Section~\ref{sec:loss-driven reward} for details).

Next, we illustrate how these components interact seamlessly in the iterative training workflow.

\begin{figure}[t!]
\includegraphics[width=0.95\columnwidth]{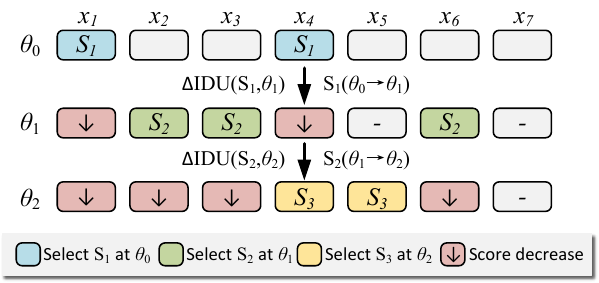}
  \caption{Iterative Sample Selection Guided by IDU Scores.}
  \label{fig:idu}
  \vspace{-1em}
\end{figure}

\subsection{Training Iteration Workflow of \model}
\label{sub:workflow}

The \model integrates iterative data selection with LLM instruction tuning. Each training iteration $t$ within \model comprises four steps.

\stitle{Step 1: Difficulty-Aware Cluster Selection.}  
    Select the optimal coarse-level difficulty cluster \( C_{i^*} \) via the MAB-EXP3 algorithm, guided by the reward derived from previous training iterations, reflecting the cluster's historical effectiveness.

\stitle{Step 2: Fine-Grained Sample Selection.}  
    Within the cluster \( C_{i^*} \), utilize the IDU function to select the top \( n_{i^*} \) most informative samples. These samples form the training subset \( S_t \). For example, in Figure~\ref{fig:idu}, at iteration \(\theta_0\), samples with the highest initial IDU scores (labeled as \(S_1\)) are chosen for training.

\stitle{Step 3: LLM Instruction Tuning.}  
    The selected samples (\(S_t\)) are used to fine-tune the model parameters, transitioning from the current parameters \(\theta_{t-1}\) to the updated parameters \(\theta_t\).

\stitle{Step 4: Reward and Utility Updates.}  
    After fine-tuning, trained samples typically show decreased IDU scores, reflecting reduced informativeness. This reduction serves as the training reward. As shown in Figure~\ref{fig:idu}, lowered IDU scores of previously selected samples (\eg \(S_1\) at \(\theta_0\) and \(S_2\) at \(\theta_1\)) prompt dynamic selection of new, more informative samples for subsequent iterations (\eg \(S_2\) to \(S_3\)).  
    Finally, both IDU scores and the MAB weights are updated accordingly, guiding the sample selection process in future iterations

Through this structured workflow, \lead~continuously and adaptively selects the most beneficial samples at each training step.
\section{The Design Details of \lead}
\label{sec:details}

We first show how to optimize our IDU utility under a budget constraint (Section~\ref{sec:Instance-Level Dynamic Uncertainty}), followed by an adaptive data selection scheduler via MAB algorithms (Section~\ref{sub:mab}), and finally, a self-guided IDU-based reward for cluster evaluation (Section~\ref{sec:loss-driven reward}).

\subsection{Instance-Level Dynamic Uncertainty Optimization under the Budget Constraint} 
\label{sec:Instance-Level Dynamic Uncertainty}

In Section~\ref{sec:pre_idu}, we introduced the $IDU$ utility (Eq.~\eqref{def:utility function}) for estimating sample utilities in iterative data selection. 
Note that our LEAD aims to iteratively select subsets of samples with the highest cumulative utility gain, defined as the expected reduction in average $IDU$ at each iteration ($\Delta IDU_t$) under a total budget constraint $B$. Formally, our optimization problem can be defined as follows.

\begin{problem}[Budget-Constrained IDU Utility Optimization] \label{def:idu} 
Given a total selection budget $B$, our goal is to maximize the cumulative expected utility over $T$ training iterations: 
\begin{align} \small
& \max_{b,T} \sum_{t=1}^{T} \mathbb{E}[\Delta IDU_t], \quad   \text{s.t.} \sum_{t=1}^{T} \mathbb{E}[n_t] \leq B \\
&\textit{where} \quad \mathbb{E}[n_t] = \alpha \cdot (1-b) \cdot |\overline{C}| \cdot (1+\text{CV}^2) \cdot (1+\mathcal{O}(\gamma))
\end{align} 

Here, \( n_t \) denotes the number of samples selected at iteration \( t \), \(\alpha\) is the sampling ratio, \(b\in[0,1)\) is the smoothing parameter controlling the influence of historical utility, \(|\overline{C}|\) is the average cluster size, and \(\text{CV}^2 = \frac{1}{K}\sum_{i=1}^{K}\frac{(|C_i|-|\overline{C}|)^2}{|\overline{C}|^2}\) quantifies variability among cluster sizes.
\end{problem}

To solve this problem, we construct a Lagrangian function incorporating the budget constraint and apply the complementary slackness condition to derive the optimal smoothing parameter \( b^* \). 
Specifically, the optimal smoothing coefficient \( b^* \) that maximizes cumulative utility gain under the budget constraint is given by:
$b^* = 1 - \frac{B}{\alpha\cdot|\overline{C}|\cdot T\cdot(1+\text{CV}^2)}$.
The detailed derivation and theoretical justification of \( b^* \) are provided in Theorem~\ref{proof:optimal smoothing coefficient} (Section~\ref{sec:proof}).


In practice, to effectively implement the optimal solution to our budget-constrained utility maximization problem, we first derive the optimal smoothing coefficient \( b^* \) from the theoretical analysis above. 
However, to fully instantiate our IDU utility function, we must also efficiently estimate the utility changes ($\Delta L'(\theta_t, S_t)$) between consecutive training iterations, as this term directly contributes to computing the cumulative utility gain $\Delta IDU_t$.
Directly calculating these utility changes would typically require additional inference steps, violating our zero-cost constraint.

To address this, we introduce the 
gradient-based approximation of utility change, as discussed below.

\stitle{Gradient-Based Approximation of Utility Change.}
Our approach efficiently utilizes gradient information computed during standard model training, thus requiring no extra computational resources beyond regular forward-backward propagation.

Formally, consider a subset of samples \( S_i \). When model parameters are updated from \(\theta_{t-1}\) to \(\theta_t\), the average uncertainty change (utility change) \(\Delta L(\theta_t, S_i)\) can be approximated as follows:

\begin{theorem}[Utility Change Approximation]
\label{the:iuchange}
For a given sample subset \( S_i \), the utility change from parameter update \(\theta_{t-1}\) to \(\theta_t\) can be approximated as:
\begin{align}\small
\Delta L'(\theta_t, S_i) &\equiv \frac{1}{|S_i|}\sum_{x\in S_i}(L(\theta_t, x)-L(\theta_{t-1}, x)) \nonumber\\
&\approx -\eta\left[\beta^2\delta_{t_k}+(1-\beta)^2\delta_{t-1}+2\beta(1-\beta)\sqrt{\delta_{t_k}\delta_{t-1}}\cos\phi\right],
\end{align}
where \(\eta\) is the learning rate, \(\delta_{t_k}\) and \(\delta_{t-1}\) denote historical gradient norms, and \(\phi\) is the angle between consecutive gradient directions, given by:
$
\cos\phi=\frac{\Delta\theta_{t_k}^{\top}\Delta\theta_{t-1}}{\|\Delta\theta_{t_k}\|\cdot\|\Delta\theta_{t-1}\|}.
$
\end{theorem}

This approach ensures that our utility estimation remains efficient, accurate, and fully integrated into standard model training workflows.
The complete derivation of this gradient-based approximation method is presented in Theorem~\ref{proof:iu_change_approximation} (Section~\ref{sec:proof}).


While the above approximation method significantly enhances efficiency, its accuracy critically depends on selecting an appropriate approximation coefficient $\beta$.
To further refine our method, we analytically derive the optimal approximation weight \(\beta^*\) that minimizes approximation error.

\stitle{Optimal Approximation Coefficient \(\beta^*\).}
Formally, we define the approximation error function as:
$J(\beta)=\|\Delta L(\theta_t, S_i)-\Delta L'(\theta_t, S_i)\|^2$.
Minimizing this error function leads us to the theoretical  
\(\beta^*\):
\begin{theorem}[Optimal Weight 
\(\beta^*\)] The optimal approximation weight \(\beta^*\) minimizing the error function 
\(J(\beta)\) is given by: \begin{equation}\small \label{eq:optimal_beta} \beta^*=\frac{\delta_{t-1}-\sqrt{\delta_{t_k}\delta_{t-1}}\cos\phi}{\delta_{t_k}+\delta_{t-1}-2\sqrt{\delta_{t_k}\delta_{t-1}}\cos\phi}. \end{equation} \end{theorem}

Detailed proofs and analyses regarding the derivation of this optimal coefficient are provided in Theorem~\ref{proof:iu_change_approximation} (Section~\ref{sec:proof}).

Finally, to rigorously evaluate the theoretical guarantees and practical utility of our gradient-based approximation, we establish a formal approximation error bound as follows.






\stitle{Approximation Error Bound.}
We bound the approximation error between the approximated loss $L'$ and the true loss $L$.

\begin{theorem}[Approximation Error Bound]
\label{theorem:loss error bound}
With the optimal weight \(\beta^*\), the error between the approximated loss \(L'\) and the true loss \(L\) satisfies:
\[ \|L'(\theta_t, x)-L(\theta_t, x)\|\leq\epsilon_{\text{taylor}}+\epsilon_{\text{approx}}~,\]

\noindent where:
\begin{itemize}[leftmargin=*, noitemsep]
    \item   $L'(\theta_i, x)=L(\theta_{i-1}, x)+ \Delta L'({\theta}_t, S_t)$
    \item $\epsilon_{taylor} = \frac{1}{2}\eta^2 \cdot \max_{\theta} \|\nabla^2 L(\theta, x)\| \cdot \|\nabla L(S_i, \theta_{i-1})\|^2$ is the error from Taylor expansion.
    \item $\epsilon_{approx} = \eta \cdot \|\nabla L(S_i, \theta_{i-1}) - (\beta^* \cdot \nabla L(S_{i_k}, \theta_{i_k-1}) + (1-\beta^*) \cdot \nabla L(S_{i-1}, \theta_{i-2}))\|^2$ is the error from gradient approximation.
\end{itemize}
\end{theorem}

\subsection{Adaptive Data Selection via MAB-Integrated Training Scheduler}
\label{sub:mab}

In this section, we propose a novel training scheduler for the \model framework that integrates the Multi-Armed Bandit (MAB) algorithm with our IDU utility function. The scheduler adaptively selects training data clusters based on their evolving informativeness.

\stitle{Step 1: Difficulty-Aware Cluster Selection}. Initially, we set the weights $ W= \{w_1, w_2, ...., w_{K}\}$ for all clusters categorized by difficulty level, where $w_i$ denotes the weight of cluster $C_i$ and $K$ is the number of clusters. To assess the difficulty score of each cluster, we employ the EXP3~\cite{auer2002nonstochastic} algorithm, a well-established method within the MAB framework, for the cluster selection. Specifically, for each iteration $t$, we first calculate the cluster score $DC_t(i)$ of the cluster $C_i$ based on the cluster weight $w_i$, and then select a cluster (arm) $DC_t^*$ with the highest score $DC$. The $DC_t(i)$ can be computed as:

\begin{equation}\small \label{equation_exp_pi}
DC_t(i) = \left(1 - \gamma\right) \frac{w_i^{(t)}}{\sum_{j=1}^{K} w_j^{(t)}} + \frac{\gamma}{K}
\end{equation}

\noindent where $\gamma$ controls the exploration-exploitation trade-off. 

The selected cluster at iteration 
$t$ is the one with the highest probability: $C_{i^*} = \arg\max_{i \in [1,K]} DC_t(i)$.




\stitle{Step 2: Sample Selection with IDU.} After selecting a cluster {$C_i$} with the highest $DC$ score, we apply our previously introduced IDU utility function to sample the most informative subset $B_{C_i}$ within the selected cluster {$C_i$}. {Specifically, we select samples with the highest IDU scores to maximize utility gain at each iteration.}                 

\stitle{Step 3: Model Training and Reward Computation.} 
Using the selected subset $B_{C_i}$, we train the large language model during iteration $t$. Once training is complete, we compute a reward $r_i^{(t)}$ to quantify the model's improvement resulting from the selected samples (Please refer to Section~\ref{sec:loss-driven reward} for details).

\stitle{Step 4: Cluster Weight Updates for Next Round Selection.}
After obtaining the reward $r_i^{(t)}$, we update the cluster weights $w_i^{(t+1)}$ according to EXP3 update rule: 

\begin{equation} \small \label{equation_exp_wi}
      w_i^{(t+1)} = \begin{cases}
        w_i^{(t)} \exp\left(\frac{\gamma}{K} \frac{r^{(t)}}{DC_t{(i)}}\right), & i = i_t \\
        w_i^{(t)}, & \text{otherwise} \end{cases}
\end{equation}

This adaptive weight-update mechanism ensures clusters that consistently yield high utility are progressively favored in subsequent iterations, achieving adaptive training data selection.


\subsection{Self-Guided IDU-Based Reward}
\label{sec:loss-driven reward}


An effective reward function is critical to guiding effective cluster selection within the MAB framework. Ideally, such a reward should precisely capture each cluster's direct contribution to model improvement, while remaining computationally efficient and fully integrated into the training process.

To achieve this, we propose a \textit{Self-Guided IDU-Based Reward}, leveraging our previously defined IDU utility to efficiently quantify each cluster's contribution to model improvement without additional inference overhead.
%
Formally, the reward for training on cluster $C_i$ at iteration $t$ is computed as:

\begin{equation} \small
r_i^{(t)}= InfoGain(C_i, t) = \mathbb{E}_{x_i \in C_i}\left[IDU(\theta_{t-1}, x_i) - IDU(\theta_{t}, x_i)\right],
\end{equation}

\noindent where $\theta_{t-1}$ and $\theta_{t}$ represent the model parameters before and after training, respectively. To maintain numerical stability and consistent scaling, rewards are further normalized to the range $[-1, 1]$ via min-max normalization.

Compared to traditional reward designs~\cite{chai2022selective}, our self-guided reward naturally integrates into the standard training loop, accurately reflects dynamic model improvements at no additional inference cost, and significantly simplifies the reward computation.

\section{Theoretical Guarantees}
\label{sec:proof}

In this section, we analyze the theoretical guarantees of our IDU utility and the \model framework.

\subsection{Optimal Smoothing Coefficient} \label{sec:idu_approx}
We now analyze the optimal smoothing coefficient for the budget-constrained IDU optimization (\textsc{Problem}~\ref{def:idu}, presented in Section~\ref{sec:Instance-Level Dynamic Uncertainty}).



\begin{theorem}[\textbf{Optimal Smoothing Coefficient}]
\label{proof:optimal smoothing coefficient}
The optimal smoothing coefficient $b^*$ that maximizes the cumulative utility gain under the budget constraint is:
\begin{equation}\small
b^* = 1 - \frac{B}{n_0 T \cdot (1 + \text{CV}^2)}
\end{equation}

\noindent where $n_0 = \alpha \cdot \overline{|C|}$ is the expected batch size without smoothing and heterogeneity effects, $B$ is the total budget, $T$ is the number of training steps, and $\text{CV}^2$ quantifies cluster size variability.
\end{theorem}


Under a total budget $B$, we propose the optimization problem:
\begin{align} \small
& \max_{b,T} \sum_{t=1}^{T} \Delta IDU_t, \quad   \text{s.t.} \sum_{t=1}^{T} n_t \leq B 
\end{align} 

The overall goal is to maximize the cumulative utility gain, and the cumulative utility gain depends on the  $\Delta IDU_t(x)$ of each round. \begin{align} \small R^{(t)} = \Delta IDU_t &= \sum_{x \in S_t} \left(IDU(\theta_t, x) - IDU(\theta_{t-1},  x)\right) \end{align} 

We take $\Delta IDU_t(x)$ of each round as the reward of the current round to guide the selection of new groups in the next round.

As the selection rounds typically exceed 5, the utility-based reward for cluster \( C_t \) simplifies to:


\begin{equation} \label{euq_IDUt}
    R^{(t)} = \Delta IDU_t = -(1-b)\eta_t|S_t|\Psi_t. 
\end{equation}

The specific simplification process can be referred to as Lemma~\ref{lemma:Utility Change Decomposition}.
Here $\Delta IDU_t$ depends on the size of $|S_t|=n_t$. Therefore, before estimating $\Delta IDU_t$, we need to estimate $n_t$. We get it in four steps.

\stitle{Step 1: Estimate sample size selected in the t-th round $n_t$}.
The probability of all clusters being selected in the initial round is the same, so the clusters are randomly selected in the first round. According to the Eq.~\eqref{equation_exp_wi} and Eq.~\eqref{equation_exp_pi}, which cluster is selected in the next round depends on which cluster was selected in the previous round. 
So we can only estimate the expectation of $n_t$.
Then $\mathbb{E}[n_t]$ can be simplified as follows (see Lemma~\ref{proof:expected sample size} for details):

\begin{equation}\small \label{equ_Ent}
\mathbb{E}[n_t] = \alpha \cdot (1-b) \cdot \frac{\sum_{i=1}^K |C_i|^2}{\sum_{i=1}^K |C_i|} \cdot \left(1 + \mathcal{O}(\gamma)\right)
\end{equation}

\stitle{Step 2: Estimate the expectation of utility gain  $\Delta IDU_t$.}
Since the utility gain  $\Delta IDU_t$ in the $T-th$ round depends on nt, and for $n_t$, due to the randomness of the MAB when selecting the cluster, we can only estimate the expectations. Therefore, it is necessary to further solve the expectations of  $\Delta IDU_t$. According to the Eq.~\eqref{euq_IDUt} and Eq.~\eqref{equ_Ent}, we can further obtain $\mathbb{E}[\Delta IDU_t]$.
\begin{align}\small
\sum_{t=1}^T \mathbb{E}[\Delta IDU_t] &= -\sum_{t=1}^T (1-b)\eta_t \cdot \mathbb{E}[|S_t|\Psi_t] \\
&= -n_0 \cdot (1-b)^2 \cdot (1 + \text{CV}^2) \cdot \sum_{t=1}^T \eta_t \delta_t \small
\end{align}
where $n_0 = \alpha \cdot \overline{|C|}$ represents the expected sample size without smoothing,  $\mathbb{E}[\Psi_t \cdot |S_t|] = \delta_t \cdot \mathbb{E}[n_t]$, $\delta_t$ represents the average per-sample utility contribution. 

\stitle{Step 3: Redefine objective and constrained condition.} Having derived the expected sample size and utility gain, we now reformulate our optimization problem by incorporating these expectations. 

\begin{align} \small
& \max_{b,T} \sum_{t=1}^{T} \mathbb{E}[\Delta IDU_t], \quad   \text{s.t.} \sum_{t=1}^{T} \mathbb{E}[n_t] \leq B \\
&\textit{where} \quad \mathbb{E}[n_t] = \alpha \cdot (1-b) \cdot |\overline{C}| \cdot (1+\text{CV}^2) \cdot (1+\mathcal{O}(\gamma))
\end{align} 

Let $\bar{\eta} \delta = \frac{1}{T}\sum_{t=1}^T \eta_t \delta_t$,
The budget constraint becomes:
\begin{align}\small
\sum_{t=1}^T \mathbb{E}[n_t] &= \sum_{t=1}^T n_0 \cdot (1-b) \cdot (1 + \text{CV}^2) \leq B \small 
\end{align}

\stitle{Step 4: Solving optimal $b^*$ and $T^*$.}
We formulate the Lagrangian:
\begin{align}
\small
\mathcal{L}(b,\lambda) &= \mathbb{E}[\Delta IDU_t] - \lambda(\mathbb{E}[n_t]-B) \\
&= -n_0 \cdot T \cdot \bar{\eta} \delta \cdot (1-b)^2 \cdot (1 + \text{CV}^2) + \\ & \lambda(n_0 \cdot T \cdot (1-b) \cdot (1 + \text{CV}^2) - B)
\end{align}
Taking the partial derivative with respect to $b$ and setting it to zero:
\begin{align}\small
\frac{\partial \mathcal{L}}{\partial b} = 0 
&\Rightarrow 2\bar{\eta} \delta \cdot (1-b) = \lambda
\end{align}

The complementary slackness condition states $\lambda(n_0 \cdot T \cdot (1-b) \cdot (1 + \text{CV}^2) - B) = 0$. Since $\lambda \neq 0$ (as verified by the optimality condition), the budget constraint must be tight:
\begin{equation}\small
n_0 \cdot T \cdot (1-b) \cdot (1 + \text{CV}^2) = B \Rightarrow
b^* = 1 - \frac{B}{n_0 \cdot T \cdot (1 + \text{CV}^2)}
\end{equation}

We require $0 \leq b^* < 1$, which implies:
\begin{equation} \small
T_{\min} = \left\lceil\frac{B}{n_0 \cdot (1 + \text{CV}^2)}\right\rceil + 1 
\end{equation}

\begin{lemma}[\textbf{Batch Utility Change Decomposition}]
\label{lemma:Utility Change Decomposition}
The utility change for batch $S_t$ under the smoothed utility function can be expressed as:
\begin{equation}\small
\Delta IDU_t = 
\begin{cases} 
-(1-b)\eta_t |S_t| \Psi_t + b |S_t| \delta_{t-1}(1 - b^{t-1}), & t \leq 5 \\
-(1-b)\eta_t |S_t| \Psi_t, & t > 5 
\end{cases}
\end{equation}
where $\Psi_t$ denotes the gradient alignment term:
\begin{equation}\small
\Psi_t = \beta_t^2 \delta_{t_k} + (1-\beta_t)^2 \delta_{t-1} + 2\beta_t(1-\beta_t)\sqrt{\delta_{t_k}\delta_{t-1}}\cos\phi_t
\end{equation}
\end{lemma}

\begin{proof}
For any $x \in S_t$, $\Delta IDU_t(x)$ can be decomposed as:
\begin{align}\small
\Delta IDU_t(x) &= (1-b)\Delta L(\theta_t, x) + b(1-b)\sum_{k=0}^{t-3} b^{k} \Delta L(\theta_{t-2-k}, x) \nonumber \small \\
&\quad + (1-b)b^{t-1}IDU(\theta_0, x) \small 
\end{align}


For the historical cumulative terms when $t \leq 5$, we apply finite-order approximation:
\begin{equation}\small
\sum_{k=0}^{t-3} b^{k} \Delta L(\theta_{t-2-k}, x) \approx \delta_{t-1}\frac{1 - b^{t-2}}{1 - b}
\end{equation}

The initial utility term $IDU(\theta_0, x)$ becomes a constant $C_0$ after aggregation. Summing over batch $S_t$ gives:
\begin{align}\small
\Delta IDU_t &= -(1-b)\eta_t |S_t|\Psi_t + b|S_t|\delta_{t-1}(1-b^{t-1}) \nonumber \\
&\quad + (1-b)b^{t-1}|S_t|C_0
\end{align}

When $t > 5$, the exponential decay term $b^{t-1}$ becomes negligible:
\begin{equation}\small
\Delta IDU_t \approx -(1-b)\eta_t |S_t|\Psi_t \quad 
\end{equation}
\end{proof}

\begin{lemma}[\textbf{Expected Sample Size Under MAB mechanism}]
\label{proof:expected sample size}
In the MAB framework using EXP3 for cluster selection with smoothed utility, the expected sample size per round $\mathbb{E}[n_t]$ satisfies:
\begin{equation}\small
\mathbb{E}[n_t] = \alpha \cdot (1-b) \cdot \overline{|C|} \cdot (1 + \text{CV}^2) \cdot \left(1 + \mathcal{O}(\gamma)\right)
\end{equation}
where $\alpha$ is the sampling rate, $b$ is the smoothing coefficient, $|C_i|$ is the size of cluster $i$, and $\gamma$ is the exploration rate in function \ref{equation_exp_pi}.
\end{lemma}

\begin{proof}
We analyze cluster selection probabilities in the EXP3 algorithm when used with our smoothed utility rewards. The reward signal for selecting cluster $i$ at time $t$ is:
\begin{equation}\small
R_i^{(t)} = \Delta IDU_t \propto (1-b)|C_i|
\end{equation}

This relationship follows directly from Lemma \ref{lemma:Utility Change Decomposition}. Since $|S_t|$ is proportional to cluster size $|C_i|$ when cluster $i$ is selected, and assuming $\Psi_t$ and $\eta_t$ are approximately constant across clusters, we derive $R_i^{(t)} \propto (1-b)|C_i|$.

From the weight update Eq.~\eqref{equation_exp_pi} and Eq.~\eqref{equation_exp_wi} in the MAB EXP3 algorithm.
As the algorithm converges to steady state, the weights stabilize such that:
\begin{equation}\small
\frac{w_i^{(t)}}{\sum_{j=1}^K w_j^{(t)}} \propto \exp\left(\sum_{\tau=1}^{t-1}\frac{\gamma}{K}\frac{R_i^{(\tau)}}{p_i^{(\tau)}}\right)
\end{equation}


In the fully converged regime, assuming small $\gamma$ and $\epsilon$, and sufficiently heterogeneous cluster sizes, we can derive a fixed-point equation. At this fixed point, the ratio $\frac{R_i^{(t)}}{p_i^{(t)}}$ becomes approximately constant across arms, leading to:
\begin{equation}\small
p_i^{(t)} \approx \frac{(1-\gamma)(1-b)|C_i|}{\sum_{j=1}^K (1-b)|C_j|} + \frac{\gamma}{K} \approx \frac{(1-b)|C_i|}{\sum_{j=1}^K |C_j|} + \mathcal{O}(\gamma)
\end{equation}


The expected sample size in round $t$ is:
\begin{align}\small
\mathbb{E}[n_t] &= \alpha \sum_{i=1}^K p_i^{(t)} |C_i| = \alpha (1-b) \frac{\sum_{i=1}^K |C_i|^2}{\sum_{j=1}^K |C_j|} + \alpha \cdot \mathcal{O}(\gamma) \sum_{i=1}^K |C_i| 
\end{align}

Since $\sum_{i=1}^K |C_i| = N$ (total dataset size), we can express this as:
\begin{equation}\small
\mathbb{E}[n_t] = \alpha \cdot (1-b) \cdot \frac{\sum_{i=1}^K |C_i|^2}{\sum_{i=1}^K |C_i|} \cdot \left(1 + \mathcal{O}(\gamma)\right)
\end{equation}

Let $\overline{|C|} = \frac{1}{K}\sum_{i=1}^K |C_i|$ be the average cluster size. Using the relation between variance and second moment:


Substituting into our expected sample size formula:
\begin{equation}\small
\mathbb{E}[n_t] = \alpha \cdot (1-b) \cdot \overline{|C|} \cdot (1 + \text{CV}^2) \cdot \left(1 + \mathcal{O}(\gamma)\right)
\end{equation}

\end{proof}

\subsection{\mbox{Loss Changes in Gradient-Based Approximation}} \label{sec:loss_change}




Recap that we have introduced utility function Eq.~\eqref{def:utility function} in Section~\ref{sec:pre_idu}, In this section, we try to approximate the loss reduction $\Delta L'(\theta_t,x)$.

\begin{theorem}[\textbf{IU Change Approximation}]
\label{proof:iu_change_approximation}
For any sample set $S_t$, the average uncertainty change $\Delta L'(\theta_t, S_t)$ when model parameters update from $\theta_{t-1}$ to $\theta_t$ can be approximated as:
\begin{align} \small
    \delta_t &\equiv\Delta L'(\theta_t, S_t) \\
    &= -\eta\Big[\beta^2\delta_{t_k} + (1-\beta)^2\delta_{t-1}  + 2\beta(1-\beta)\sqrt{\delta_{t_k}\delta_{t-1}}\cos\phi\Big]
\end{align}
where $\phi$ is the angle between parameter update directions $\Delta\theta_{t_k}$ and $\Delta\theta_{t-1}$, with $\cos\phi = \frac{\Delta\theta_{t_k}^{\top}\Delta\theta_{t-1}}{\|\Delta\theta_{t_k}\|\|\Delta\theta_{t-1}\|}$.

\begin{equation} \small
\beta^* = \frac{\delta_{t-1} - \sqrt{\delta_{t_k}\delta_{t-1}}\cos\phi}{\delta_{t_k} + \delta_{t-1} - 2\sqrt{\delta_{t_k}\delta_{t-1}}\cos\phi}
\end{equation}
\end{theorem}

\stitle{Step 1: Simplify the loss change.} 
Assume at iteration $t$, model parameters are updated via gradient descent:
$\theta_t = \theta_{t-1} - \eta_t\nabla L(S_t, \theta_{t-1})$, where $\nabla L(S_t, \theta_{t-1}) = \frac{1}{|S_t|}\sum_{x\in S_t}\nabla L(x, \theta_{t-1})$ is the average gradient of subset $S_t$.
For each sample $x \in S_t$, the loss function $L(\theta, x)$ is expanded using first-order Taylor expansion at $\theta_{t-1}$:
\begin{equation}\small
    L(\theta_t, x) \approx L(\theta_{t-1}, x) + \nabla L(\theta_{t-1}, x)^{\top}(\theta_t - \theta_{t-1})
\end{equation}


Averaging over all samples in $S_t$:
\begin{align}\small
 \delta_t=\Delta L'(\theta_t, S_t) 
    &\approx -\eta_t\frac{1}{|S_t|}\sum_{x\in S_t}\nabla L(\theta_{t-1}, x)^{\top}\nabla L(\theta_{t-1}, S_t)\\
    &= -\eta_t\|\nabla L(\theta_{t-1}, S_t)\|^2
\end{align}
It can be concluded that the loss reduction is related to the gradient.

\stitle{Step 2: Approximate the gradient.} 
To further approximate the loss, we need to approximate the gradient. 
Here we consider that the gradient at the current moment is related to the gradient at the previous moment and the gradient when the cluster used at the current moment was first selected.
\begin{equation} \small
    \nabla L'(S_t,\theta_{t-1}) \equiv  \beta \cdot \nabla L(S_{t_k}, \theta_{t_k-1}) + (1-\beta) \cdot \nabla L(S_{t-1},\theta_{t-2}),
\end{equation} where $t_k$ is the most recent step when $C_k$ was previously selected, $C_k$ is the cluster selected at step $t$, where $\beta \in [0,1]$ is a weighting coefficient measuring the relative importance of cluster-specific historical information versus recent optimization direction.

\stitle{Step 3: Solving optimal $\beta^*$ to obtain final IU Change Approximation $\Delta L'(\theta_t, S_t)$.}
The $\beta^*$ can be solved by minimizing the difference between the current gradient and the approximate gradient.
\begin{equation} \small
J(\beta) = \|\nabla L_t - (\beta \nabla L_{t_k} + (1-\beta)\nabla L_{t-1})\|^2
\end{equation}

Using the gradient descent update rule $\Delta\theta_t = -\eta\, \nabla L_t$, we rewrite in terms of parameter updates:
\begin{align} \small
J(\beta) &= \frac{1}{\eta^2}\left\|\Delta\theta_t - \Bigl(\beta\, \Delta\theta_{t_k} + (1-\beta)\, \Delta\theta_{t-1}\Bigr) \right\|^2.
\end{align}

Since $\|\Delta\theta_{t_k}\|^2 \approx -\eta\, \delta_{t_k}$ and $\cos\phi = \frac{\Delta\theta_{t_k}^{\top}\Delta\theta_{t-1}}{\|\Delta\theta_{t_k}\|\|\Delta\theta_{t-1}\|}$.

Setting $\frac{d\tilde{J}}{d\beta}=0$ yields the optimal coefficient:
\begin{equation} \small
\beta^* = \frac{\delta_{t-1} - 
\sqrt{\delta_{t_k}\delta_{t-1}}\, \cos\phi}{\delta_{t_k} + \delta_{t-1} - 2\sqrt{\delta_{t_k}\delta_{t-1}}\, \cos\phi}.
\end{equation}

The loss change is then approximated as:
\begin{align} \small
\delta_t &= -\eta\Bigl[(\beta^*)^2\, \delta_{t_k} + (1-\beta^*)^2\, \delta_{t-1} + 2\,\beta^*(1-\beta^*)\,\sqrt{\delta_{t_k}\delta_{t-1}}\,\cos\phi\Bigr]. 
\end{align}

\section{Experiments}
\label{sec:expr}

\subsection{Experimental Setup}




\stitle{Data Pool.} 
To simulate realistic and diverse training scenarios, we construct a large-scale and heterogeneous data pool comprising approximately \textbf{600,000 samples}. Our dataset integrates multiple well-established public sources, including  WizardLM (ShareGPT)~\cite{luinstag}, WizardLM (Alpaca)~\cite{luinstag}, UltraChat~\cite{ding2023enhancing}, Standard Alpaca~\cite{taori2023stanford}, unnatural~\cite{honovich2022unnatural}, Alpaca code~\cite{chaudhary2023code}, MATH~\cite{hendrycks2021measuring}, GSM8K~\cite{cobbe2021training}. We closely follow Tulu~\cite{wang2023far} to process these datasets. All methods will select data from this pool for LLMs' instruction tuning.

\stitle{Benchmarks and Metrics.} We comprehensively evaluate our method across four representative tasks that reflect critical capabilities required by modern LLMs. 


\begin{itemize}
\item \texttt{Code Generation.} 
We use the extensively utilized \texttt{HumanEval} benchmark~\cite{chen2021evaluating}, consisting of 164 coding problems, to evaluate the code-writing capabilities of LLMs. Performance is measured via the widely adopted \texttt{pass@10} metric.

\item \texttt{Math Reasoning.} We use \texttt{GSM8k}~\cite{cobbe2021training} to evaluate the mathematical abilities of models, which contains 1319 grade school math test data. We adopt an 8-shot setting and evaluate performance using the \texttt{exact match accuracy} metric.

\item \texttt{Multi-task Knowledge and Reasoning.} We evaluate on \texttt{MMLU}~\cite{hendrycksmeasuring}, which consists of a range of multiple-choice academic questions.
We report \texttt{accuracy} as the metric.

\item \texttt{Cross-lingual Question Answering.}  To assess multilingual understanding, we utilize the  \texttt{TYDIQA}~\cite{clark2020tydi}, featuring questions from 11 diverse languages. We report standard \texttt{F1 scores} for both passage selection and answer span extraction tasks.
\end{itemize}

\stitle{Baselines.} We study several existing state-of-the-art methods as our baselines for data selection.

\etitle {(1) Full Data}: Train the model using the entire data pool.

\etitle {(2) Random Selection}~\cite{xia2024rethinking}: Randomly selects training samples.

\etitle {(3) Instruction-Following Difficulty (IFD)}~\cite{li2024quantity}: 
Selects samples based on a complexity metric measuring instruction-following difficulty.

\etitle {(4) Perplexity (PPL)}~\cite{li2024superfiltering}: 
Prioritizes uncertain samples with high perplexity.

\etitle {(5) K-Center-Greedy (KCG)}~\cite{sener2018active}: 
Maximizes diversity by iteratively choosing the sample farthest from the current selection.

\etitle {(6) SelectIT}~\cite{liu2024selectit}: 
Selects samples via uncertainty-aware self-reflection during instruction tuning.
 
\etitle {(7) Token Length (TL)}~\cite{xia2024rethinking}: 
Selects samples with the longest response lengths.

\etitle {(8) ZIP}~\cite{yin2024entropy}: 
prompting a strong LLM to estimate and select samples based on quality, relevance, and complexity scores.

\begin{figure*}[t!]
\small
    \centering
        \centering
        \captionof{table}{Comparison of Performance across Different Benchmarks for Various Methods.}
        \vspace{-1em}
        \resizebox{\linewidth}{!}{%
            \begin{tabular}{l|*{9}{c}|c}
                \toprule
                \multirow{2}{*}{\textbf{Benchmark}} & \multicolumn{9}{c|}{\textbf{Methods}} & \multirow{2}{*}{$\mathbf{\Delta}$} \\
                \cmidrule(lr){2-10}
                 & \textbf{Full Data} & \textbf{Random} & \textbf{PPL} & \textbf{KCG} & \textbf{TL} & \textbf{IFD} & \textbf{SelectIT} & \textbf{ZIP} & \textbf{\model (Ours)} & \\
                \midrule
                \multicolumn{10}{c}{\textbf{LLaMA3.1-8B}} \\
                \midrule
                MMLU & 65.13 & 64.30 & 63.27 & 61.39 & 64.10 & 64.48 & 64.93  & 63.45 & \textbf{65.40} & \cellcolor{green!50!black!20}+1.10 \\
                TYDIQA & 50.94 & 40.91 & 41.89 & 43.12 & 46.47 & 55.66 & 61.33  & 45.41  & \textbf{63.24} &\cellcolor{green!50!black!20}+22.33 \\
                GSM8K & 56.63 & 54.80 & 56.32 & 51.73 & 54.28 & 43.52 & 54.89 & 57.32  & \textbf{60.88} &\cellcolor{green!50!black!20}+6.08 \\
                HumanEval & 68.52 & 70.24 & 71.44 & 69.80 & 73.99 & 70.40 & 69.33  & 67.68 & \textbf{76.95} &\cellcolor{green!50!black!20}+6.71 \\
                \textbf{Average} & \cellcolor{blue!50!black!10}60.31 & 57.66 & 58.23 & 56.51 & 59.71 & 58.52 & \cellcolor{blue!50!black!20}\underline{\textbf{62.62}}  & 58.47 &\cellcolor{blue!50!black!30}\textbf{66.62} & \cellcolor{green!50!black!20}+8.96 \\
                \midrule
                \multicolumn{10}{c}{\textbf{Mistral-7B}} \\
                \midrule
                MMLU & 61.45 & 61.68 & 62.38 & 61.02 & 61.93 & 61.65 & 64.93  & 61.93 & \textbf{62.10} &\cellcolor{green!50!black!20}+0.42 \\
                TYDIQA & 49.63 & 38.02 & 52.72 & 39.79 & 39.88 & 41.41 & 36.79  & 42.04 & \textbf{67.17} &\cellcolor{green!50!black!20}+29.15 \\
                GSM8K & 40.56 & 33.51 & 22.82 & 33.89 & 37.76 & 31.77 & 35.86 & 41.17  & \textbf{45.26}  &\cellcolor{green!50!black!20}+11.75 \\
                HumanEval & 58.37 & 57.35 & 54.68 & 59.96 & 60.54 & 52.05 & 58.15 & 61.91 & \textbf{59.01} &\cellcolor{green!50!black!20}+1.66 \\
                \textbf{Average} & \cellcolor{blue!50!black!20}\underline{\textbf{52.50}} & 47.64 & 48.15 & 48.67 & 50.03 & 46.72 & 48.93  & \cellcolor{blue!50!black!10}51.76 &\cellcolor{blue!50!black!30}\textbf{58.39}  & \cellcolor{green!50!black!20}+10.75 \\
                \midrule
                \multicolumn{10}{c}{\textbf{Qwen2-7B}} \\
                \midrule
                MMLU & 70.54 & 69.85 & 70.70 & 70.64 & 70.52 & 70.03 & 70.32  & 70.54 & \textbf{70.19} &\cellcolor{green!50!black!20}+0.34 \\
                TYDIQA & 42.94 & 43.43 & 42.63 & 40.92 & 38.91 & 35.00 & 43.80 & 34.51 & \textbf{56.06}  &\cellcolor{green!50!black!20}+12.63 \\
                GSM8K & 73.16 & 73.16 & 79.00 & 76.04 & 78.53 & 74.91 & 74.60  & 75.66 & \textbf{79.83} &\cellcolor{green!50!black!20}+6.67 \\
                HumanEval & 82.56 & 79.51 & 78.44 & 78.81 & 80.79 & 81.94 & 78.14  & 83.91 & \textbf{84.22}  &\cellcolor{green!50!black!20}+4.71 \\
                \textbf{Average} & 
                \cellcolor{blue!50!black!10}67.30 & 66.49 & \cellcolor{blue!50!black!20}\underline{\textbf{67.69}} & 66.60 & 67.19 & 65.47 & 66.72  & 66.16 &\cellcolor{blue!50!black!30}\textbf{72.58} & \cellcolor{green!50!black!20}+6.09 \\
                \bottomrule
            \end{tabular}%
        }
        \label{tab:main_result}
\end{figure*}

\stitle{Implementation Details of LEAD.} We evaluate \model using three foundational models (LLAMA-3.1-8B, Mistral-7B and Qwen2-7B) and utilize Low-Rank Adaption (LoRA)~\cite{hu2022lora} for parameter-efficient fine-tuning. The maximum learning rate is set as $2 \times 10^{-5}$ with a linear decay schedule, and the batch size is 8. We also fix the maximum input sequence length to 3080. Models are trained for 4 epochs on 4 H800 GPUs. For the MAB setting, the number of arms is set to 7. The maximum sampling budget of \model is 15$K$. 

\subsection{Exp-1: Overall Performance}


We first evaluate \model and all baseline methods using the same budget of $15K$ samples, corresponding to $2.5\%$ of the data pool.

Table~\ref{tab:main_result} summarizes the evaluation results across various benchmarks (MMLU, TYDIQA, GSM8K, and HumanEval) and model architectures (LLaMA3.1-8B, Mistral-7B, and Qwen2-7B). Overall, \model consistently outperforms state-of-the-art baselines, demonstrating its effectiveness. Note that $\Delta$ denotes the performance improvement of \model compared to the \texttt{Random} baseline.

\stitle{Consistent Effectiveness of LEAD across LLMs.} \model demonstrates remarkable effectiveness across different model architectures: For LLaMA3.1-8B, it achieves an average score of 66.62, outperforming full dataset training (60.31) by a substantial +6.31 points. Similar gains are seen with Mistral-7B (+10.75) and Qwen2-7B (+6.09). This cross-architecture consistency confirms that LEAD reliably selects high-value samples beneficial for diverse LLMs.

\stitle{2.5\% of Data is All You Need.} 
Remarkably, \model achieves these substantial gains using only 2.5\% of the entire dataset, challenging the conventional assumption that larger datasets inherently produce superior results. Specifically, our method outperforms full dataset training (\texttt{Full Data} baseline) across all model and benchmark settings. For example, on the challenging TYDIQA benchmark, our approach yields remarkable gains of $22.33$, $29.15$, and $12.63$ points of improvement across the three models, respectively, demonstrating that carefully selected instruction samples can lead to more focused and effective learning. 

\stitle{Outperforming State-of-the-art Baselines.} 
\model outperforms all baseline selection methods with consistent effectiveness across models and benchmarks. While some baselines perform well in specific cases (\eg SelectIT on LLaMA3.1-8B and PPL on Qwen2-7B), they fall short in other settings. In contrast, our approach maintains consistent high performance across the board,
Notably, on the HumanEval benchmark for code generation, \model achieves top performance across all models.
\begin{figure*}[t!]
  \centering
  \includegraphics[width=\linewidth]{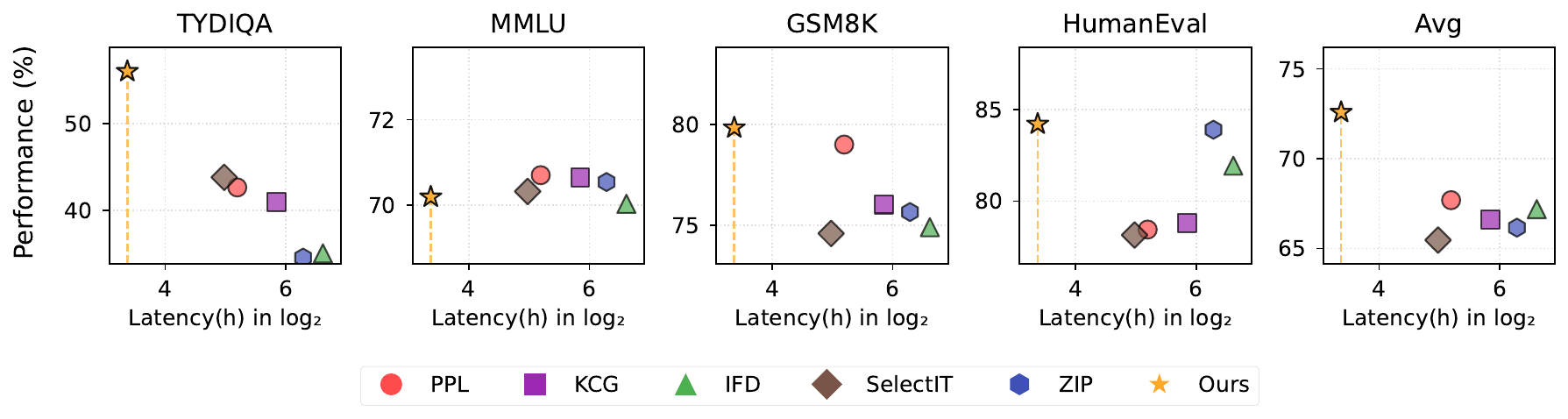}
  \caption{Comparison of Performance($y$-axis) and Latency($x$-axis) across six data selection methods. }
  \label{fig:latency}
\end{figure*}

\subsection{Exp-2: The Efficiency of \model}

We evaluate the efficiency of \model compared to baseline methods (PPL, KCG, IFD, SelectIT, and ZIP) across four benchmarks. Note that we exclude \texttt{Random} and \texttt{TL} from this comparison, as these methods incur minimal computational overhead and were shown to perform significantly worse in \textbf{Exp-1}. We report the overall latency of all methods with \textit{one round of selection iteration} on average.

\stitle{Exp-2.1: Performance \vs Latency.}
We compare performance and inference latency (in $log_2$ scale) across different methods. As shown in Figure~\ref{fig:latency}, \model (marked with a star) consistently achieves the best performance-latency trade-off, occupying the upper-left region of each plot. \model delivers a roughly 5× faster inference time compared to baselines, while maintaining top performance on benchmarks like TYDIQA, GSM8K, and HumanEval. 

\stitle{Exp-2.2: Analysis of Latency Composition.}
Figure~\ref{fig:infer_train_time} compares latency components (inference and training) of different methods. Inference time constitutes the primary computational bottleneck for traditional methods (\eg IFD: 98.0 hours, ZIP: 78.0 hours), due to repeated full-dataset inference at each selection iteration. In contrast, \model requires inference only once (10.3 hours) for initial selection, eliminating subsequent inference overhead via inference-free IDU estimation.

\begin{figure}[t!]
    \includegraphics[width=\columnwidth]{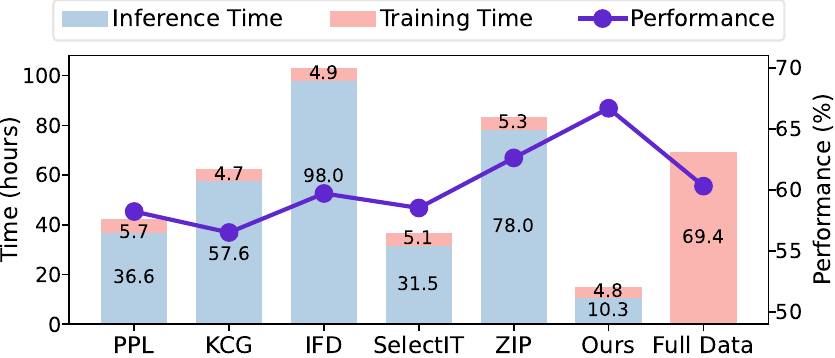}
    \caption{Inference Time (Full Data) and Training Time (Selected Data) per Iteration across Different Methods.}
    \label{fig:infer_train_time}
\end{figure}

\subsection{Exp-3: Static \vs Iterative Data Selection}

These experiments validate the necessity of iterative data selection.


\stitle{Exp-3.1: Dynamics of Sample Utility over Training.} 
We first track the overlap of samples initially identified as valuable (iteration 0) with the top-$k$ samples in later iterations (1, 4, 7, and 10). As illustrated in Figure~\ref{fig:heatmap}, the coverage rate for $k$=15,000 increases initially (from 0.77 to 0.98 at iteration 4), but significantly declines (to 0.67) in later iterations. This clearly demonstrates the dynamic nature of sample utility, emphasizing the importance of continuously adapting data selection to the evolving state of the model.

\begin{figure}[t!]
    \centering
        \centering        \includegraphics[width=0.65\columnwidth]{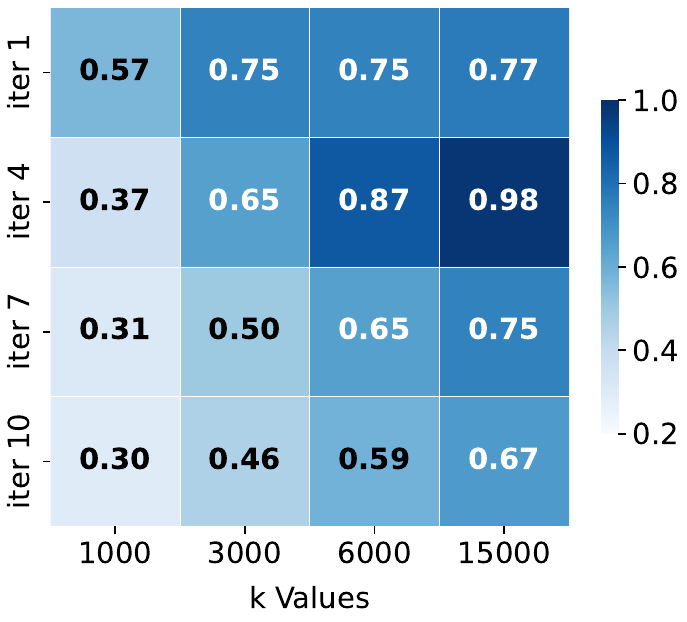}
        \vspace{-1em}
        \caption{Coverage of Top-$k$ Samples between Iter. $t$ and 0.}
        \label{fig:heatmap}
\end{figure}

\stitle{Exp-3.2: Performance of Static and Iterative Selection.}
We further compare the performance between one-round (static) and iterative selection strategies (Table~\ref{tab:iu_idu_comparison}).
Iterative LEAD (IU) consistently surpasses One-round LEAD (IU), achieving an average improvement of 1.17 points (64.33 \vs 63.16). This performance gap confirms that iterative data selection is essential, as the utility of training samples dynamically changes throughout model training.

\begin{table}[t!]
    \centering
    \renewcommand{\arraystretch}{1}
    \setlength{\tabcolsep}{7pt}
    \caption{Comparison between IU and IDU. LEAD (IDU) refers to our method using IDU as the utility function for calculating sample utility. One-round and Iterative LEAD (IU) denote non-iterative and iterative variants of the IU approach.}
    \resizebox{\linewidth}{!}{
    \newcommand{\deltavalue}[1]{\begin{tabular}{c}#1\\\cellcolor{red!15}{\large$\downarrow$#1}\end{tabular}}
    \begin{tabular}{l|cccc|c}
        \toprule
        \multirow{2}{*}{\textbf{Method}} & \multicolumn{4}{c|}{\textbf{Benchmarks}} & \multirow{2}{*}{\textbf{Average}} \\
        \cmidrule(lr){2-5}
        & \textbf{MMLU} & \textbf{TYDIQA} & \textbf{GSM8K} & \textbf{HumanEval} & \\
        \midrule
        \textbf{LEAD (IDU)} & \textbf{65.40} & \textbf{63.24} & \textbf{60.88} & \textbf{76.95} & \textbf{66.62} \\
        \midrule
        One-round LEAD (IU) & 
            \begin{tabular}{c}63.92\\\cellcolor{red!15}{\large$\downarrow$1.48}\end{tabular} & 
            \begin{tabular}{c}59.13\\\cellcolor{red!15}{\large$\downarrow$4.11}\end{tabular} & 
            \begin{tabular}{c}57.47\\\cellcolor{red!15} {\large$\downarrow$3.41}\end{tabular} & 
            \begin{tabular}{c}72.13\\\cellcolor{red!15}{\large$\downarrow$4.82}\end{tabular} & 
            \begin{tabular}{c}63.16\\\cellcolor{red!15}{\large$\downarrow$3.46}\end{tabular} \\
        \midrule
        Iterative LEAD (IU) & 
            \begin{tabular}{c}64.72\\\cellcolor{red!15}{\large$\downarrow$0.68}\end{tabular} & 
            \begin{tabular}{c}60.15\\\cellcolor{red!15}{\large$\downarrow$3.09}\end{tabular} & 
            \begin{tabular}{c}57.99\\\cellcolor{red!15}{\large$\downarrow$2.89}\end{tabular} & 
            \begin{tabular}{c}74.46\\\cellcolor{red!15}{\large$\downarrow$2.49}\end{tabular} & 
            \begin{tabular}{c}64.33\\\cellcolor{red!15}{\large$\downarrow$2.29}\end{tabular} \\
        \bottomrule
    \end{tabular}
    }
    \label{tab:iu_idu_comparison}
    \vspace{-1em}
\end{table}




\begin{table}[t!]
    \centering
    \large 
    \caption{Ablation Study of Different Modules (LLaMA3.1-8B)}
    \label{tab:ablation_study}
    \resizebox{\linewidth}{!}{
    \begin{tabular}{c|ccc|cccc|c}
        \toprule
        \multirow{2}{*}{\textbf{Models}} & \multicolumn{3}{c|}{\textbf{Module}} & \multicolumn{4}{c|}{\textbf{Benchmarks}} & \multirow{2}{*}{\textbf{Average}} \\
        \cmidrule(lr){2-4} \cmidrule(lr){5-8}
        & \textbf{MAB} & \textbf{TC} & \textbf{IDU} & \textbf{MMLU} & \textbf{TYDIQA} & \textbf{GSM8K} & \textbf{HumanEval} & \\
        \midrule
        \multirow{6}{*}{\model} 
        & \checkmark & & & 64.83 & 59.84 & 54.81 & 72.13 & 62.90 \\
        & & \checkmark & & 62.71 & 61.31 & 51.48 & 74.25 & 62.44 \\
        & & & \checkmark & 64.13 & 61.47 & 57.92 & 74.93 & 64.61 \\
        & \checkmark & \checkmark & & 65.1 & 55.88 & 57.99 & 74.41 & 63.35  \\
        & \checkmark & & \checkmark & 64.7 & 66.46 & 55.95 & 74.46 & 65.39 \\
        & & \checkmark & \checkmark & 65.3 & 64.29 & 56.40 & 73.38 & 64.84 \\
        \midrule
        \textbf{\model (Ours)} & \checkmark & \checkmark & \checkmark & \textbf{65.40} & \textbf{63.24} & \textbf{60.88} & \textbf{76.95} & \textbf{66.62} \\
        \bottomrule
    \end{tabular}
    }
\end{table}

\begin{table}[t!]
    \centering
    \caption{Ablation Study of LEAD Framework}
    \label{tab:replace_strategy}
    \renewcommand{\arraystretch}{1}
    \setlength{\tabcolsep}{2.5pt} 
    \large
    \resizebox{\linewidth}{!}{
    \begin{tabular}{c|c|cccc|c}
        \toprule
        \multirow{2}{*}{\textbf{Method}} & \multirow{2}{*}{\textbf{Replace Strategy}} & \multicolumn{4}{c|}{\textbf{Benchmarks}} & \multirow{2}{*}{\textbf{Average}}  \\
        \cmidrule(lr){3-6}
        & & \textbf{MMLU} & \textbf{TYDIQA} & \textbf{GSM8K} & \textbf{HumanEval} & \\
        \midrule
        \multirow{2}{*}{\parbox{1.5cm}{\centering Reward Function}}
        & IFD-MAB & 65.29 & 65.31 &  51.02 & 72.13 & 63.44  \\
        & PPL-MAB & 65.52 & 67.17 & 51.71 & 72.11 & 64.13  \\
        \midrule
        \multirow{4}{*}{\centering IDU} 
        & Random & 65.10 & 55.88 & 57.99 & 74.41 & 63.35  \\
        & PPL & 64.13 & 49.40 & 52.53 & 68.17 & 59.59 \\
        & IU & 64.72 & 60.15 & 57.99  & 74.46 & 63.56  \\
        & IFD & 64.92 & 54.98 & 51.86 & 70.71 & 60.62 \\
        \midrule
        \multirow{3}{*}{\centering MAB} & Random & 64.17 & 61.46 & 55.95 & 74.00 & 63.90  \\
        & Easy2Hard & 64.73 & 60.32 & 58.98 & 71.81 & 63.96  \\
        & Hard2Easy & 64.29 & 61.96 & 56.65 & 74.54 & 64.36 \\
        \midrule
        \cellcolor{green!50!black!20}\textbf{\centering Ours} & \cellcolor{green!50!black!20}\textbf{-} & \cellcolor{green!50!black!20}\textbf{65.40} & \cellcolor{green!50!black!20}\textbf{63.24} & \cellcolor{green!50!black!20}\textbf{60.88} & \cellcolor{green!50!black!20}\textbf{76.95} & \cellcolor{green!50!black!20}\textbf{66.62}  \\
        \bottomrule
    \end{tabular}
    }
\end{table}


\subsection{Exp-4: Ablation Study of LEAD}

\stitle{Exp-4.1: Ablation Study on LEAD Components.}
To validate the effectiveness of our proposed framework, we conduct an ablation study on the LLaMA3.1-8B model by systematically removing individual modules of our \model framework. 
As shown in Table~\ref{tab:ablation_study}, removing any module leads to a performance drop: average metric decreases by 1.78 (MAB), 1.23 (TC), and 3.27 (IDU). The IDU module has the most pronounced impact, particularly on TYDIQA (-7.36), underscoring its role in identifying informative samples. Removing the TC module also degrades performance across all benchmarks, confirming the value of semantic clustering. The removal of the MAB module significantly affects performance on the challenging GSM8K (-4.48), demonstrating its role in balancing exploration and exploitation. Overall, the ablation study highlights the critical contribution of each component within the \model framework.


\stitle{Exp-4.2: The Effectiveness of IDU Utility.} To demonstrate the effectiveness of our proposed Instance-Level Dynamic Uncertainty (IDU) mechanism, we conducted comprehensive experiments examining its performance from two perspectives.

First, to verify that IDU effectively smooths the instability issues during iterative selection, we compared LEAD (IDU) against LEAD (IU) on LLaMA3.1-8B. As shown in Table~\ref{tab:iu_idu_comparison}, LEAD (IDU) consistently outperforms iterative LEAD (IU) across all benchmarks with a substantial average improvement of 3.06\% (66.62 \vs 63.56). This confirms that IDU's design—combining current loss signals and historical exponential smoothing—effectively addresses the loss instability challenge inherent in conventional utility functions.

Second, to validate IDU's superiority as a utility function, we compared it against alternative utility metrics while keeping other LEAD components intact. The results in Table~\ref{tab:replace_strategy} show that replacing IDU with conventional metrics like PPL leads to dramatic performance degradation (from 66.62 to 59.59), with particularly severe reductions on TYDIQA (-13.84\%). Even when compared to the more advanced IFD metric, IDU maintains a substantial advantage (66.62 vs. 60.62). This consistent performance advantage across diverse benchmarks highlights IDU's robustness as a selection criterion that can reliably identify valuable training samples across various domains and task structures in the iterative selection process.

\stitle{Exp-4.3: The Effectiveness of MAB Module.} To assess the MAB module's contribution, we compare it against three baselines: (1) Random-\model: random selection of difficulty-aware clusters per iteration; (2) Easy2Hard-\model: iterative training from easy to hard clusters based on difficulty scores; and (3) Hard2Easy-\model: iterative training from hard to easy. For a fair comparison, all modules except the training strategy remained consistent with the \model. 

As shown in Table~\ref{tab:replace_strategy}, our MAB training schedule significantly outperforms the other three strategies, confirming its effectiveness in dynamically balancing exploration and exploitation. By adaptively selecting difficulty-aware clusters, MAB enhances both overall performance and generalizability.

In contrast, Easy2Hard-\model yields the low score (63.96), highlighting the limitations of traditional curriculum learning in instruction tuning, as a fixed progression from easy to hard can hinder learning dynamics and lead to premature convergence. Hard2Easy-\model performs slightly better (64.36), yet still underperforms compared to MAB, indicating that prioritizing difficult clusters alone does not guarantee optimal results.

\stitle{Exp-4.4: The Effectiveness of Reward Function.} 
We assess the effectiveness of our proposed IDU-based reward by comparing it with two widely-used reward metrics: Instruction-Following Difficulty (IFD)~\cite{li2024quantity} and Perplexity (PPL)~\cite{li2024superfiltering}.

As shown in Table~\ref{tab:replace_strategy}, our IDU-based reward consistently achieves the best overall performance (average 66.62), surpassing IFD (63.44) and PPL (64.13). This demonstrates that directly measuring the reduction in instance-level dynamic uncertainty provides more effective guidance for cluster selection than traditional metrics.

\begin{figure}[t!]
    \centering
    \includegraphics[width=0.68\columnwidth]{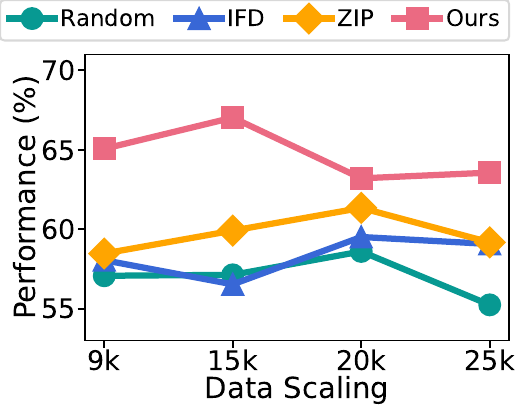}
        \vspace{-1em}
        \caption{Avg Performance by Varying Data Scaling.}
        \label{fig:avg_performance_data_scale}
\end{figure}

\begin{figure*}[t!]
    \centering
    \includegraphics[width=\linewidth]{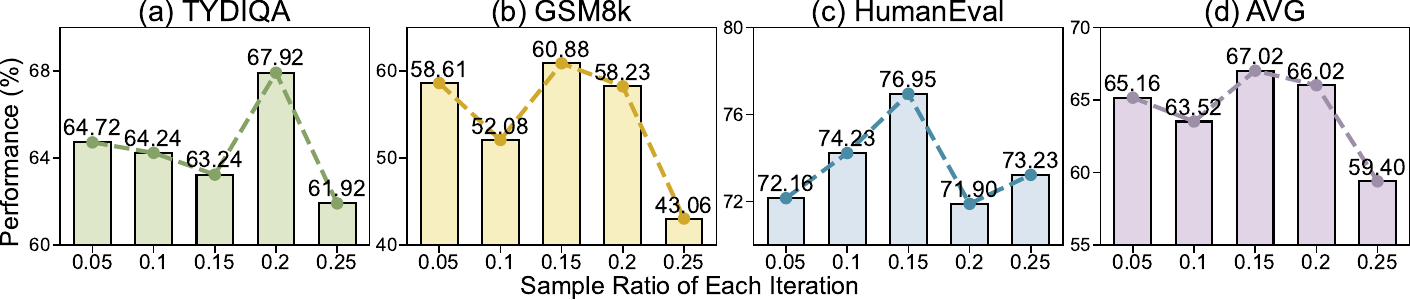}
        \vspace{-2em}
        \caption{Performance on Various Sample Ratios of Each Iteration (LLaMA3.1-8B).}
        \label{fig:sampleratio}
\end{figure*}

\begin{figure}[t!]
    \centering
    \includegraphics[width=\linewidth]{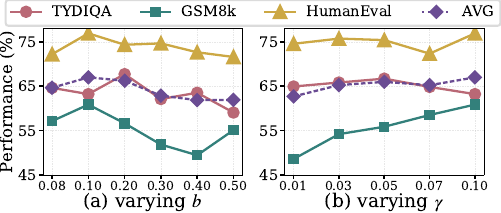}
        \vspace{-2.em}
        \caption{Parameter Sensitivity Analysis}
        \label{fig:Parameter}
\end{figure}

\subsection{Exp-5: Evaluation of Optimal Data Scaling} 
To examine the impact of data selection strategies on data scaling effectiveness, we conduct experiments using subsets with varying budgets. As illustrated in Figure~\ref{fig:avg_performance_data_scale}, \model consistently presents higher average performance than alternative selection methods across all data quantities, achieving peak performance with only 15K samples.
Notably, we observe a non-linear performance curve: gains taper and eventually decline beyond a certain data threshold,
which reveals a crucial insight: 
``alignment-suitable data'' is inherently limited. This finding challenges the conventional wisdom that more data automatically yields better results, underscoring the critical importance of strategic data selection over mere quantity.


\subsection{Exp-6: Parameter Sensitivity Analysis}
In this experiment, we conduct parameter sensitivity analysis to reveal how hyperparameters affect \model's performance across different tasks, providing insights into optimizing the framework.

\stitle{Effect of Sampling Threshold $\alpha$ of \model.} As shown in Figure~\ref{fig:sampleratio}, 
performance peaks when $\alpha$ is between 0.15 and 0.20, reaching a balance between iteration quantity and quality. Higher $\alpha$ values yield more samples per round but fewer iterations, limiting adaptability. Lower values allow more iterations but provide weaker signals. 

\stitle{Effect of Smoothing Coefficient $b$ of IDU.} 
Figure~\ref{fig:Parameter}(a) shows optimal performance at $b$=0.1, achieving a favorable trade-off between historical and current utility signals. This sweet spot effectively leverages historical information to stabilize selection while remaining responsive to recent model changes. Lower values ($b$<0.1) overemphasize current utility fluctuations, increasing susceptibility to noise.
while higher ones ($b$>0.2) overweight historical information, reducing responsiveness. 


\stitle{Effect of Exploration Rate $\gamma$ of MAB.} 
As shown in Figure~\ref{fig:Parameter}(b), the exploration-exploitation tradeoff in our MAB algorithm shows optimal performance at moderate exploration rates ($\gamma$=0.05-0.07). Minimal exploration ($\gamma$=0.01) limits discovery of new clusters, whereas excessive exploration ($\gamma$=0.12) hinders focus on promising clusters.

\section{Related Work}

\stitle{Data Selection for Instruction Tuning.} Previous works on data selection~\cite{xia2024less, zhou2023lima, hanmo2024effective,sigmod_tutorial} can be broadly categorized into two key approaches: model-agnostic methods and model-aware methods. 

Model-agnostic methods operate independently of the target model, including rule-based approaches~\cite{soldaini2024dolma,zhuo2024astraios, kopf2023openassistant, muennighoff2023octopack, cao2023instruction, luinstag,hod} that are computationally efficient but lack semantic understanding. Advanced model-based methods~\cite{lian2023slimorca, chen2023alpagasus, chenalpagasus}  like GPT-4~\cite{achiam2023gpt}  that provide nuanced assessment at high computational cost, and proxy model-based methods~\cite{li2024superfiltering, yang2024smalltolarge} that balance efficiency and quality. However, these methods cannot adapt to the specific learning characteristics of the target model. Model-aware methods~\cite{cao2023instruction, zhang2024harnessing, chai2023goodcore,chai2022selective,liu2022feature,visclean_icde,visclean_vldb} address this limitation by customizing selection based on the model's learning dynamics, though they introduce higher computational costs through required model inference or fine-tuning. In contrast, LEAD proposes a two-stage adaptive approach that efficiently combines model-aware adaptiveness with zero computational overhead, effectively addressing the challenge of balancing effectiveness and efficiency in instruction tuning data selection.

\stitle{Sample Utility Scores.} Sample utility scoring plays a critical role in data selection, employing various predefined metrics~\cite{wang2024fast, chai2023goodcore, ratner2017snorkel}. Perplexity-based metrics~\cite{marion2023less, li2024superfiltering}  favor simpler patterns, while diversity-aware selection~\cite{wu2023self, yu2024diversify} ensures broad coverage but depends heavily on pre-trained embedding quality. Quality-based metrics incorporating influence scoring~\cite{xia2024less, ghorbani2019data, kwondatainf, choe2024your} and external model~\cite{li2024one} evaluation are theoretically sound but require expensive gradient computations. Complexity-based selection~\cite{li2024quantity, liu2024what} risks including noisy samples that hinder convergence, while uncertainty-driven metrics~\cite{han2025automatic, liu2024selectit} suffer from instability due to loss landscape irregularities. A common limitation across these approaches is their significant computational overhead. Although recent efforts have improved data efficiency in utility estimation, they still incur additional costs. We propose IDU, a novel utility function achieving zero-cost estimation while maintaining selection effectiveness.

\section{Conclusion}


In this paper, we proposed \model, an efficient iterative data selection framework for instruction tuning of LLMs. \model introduces Instance-Level Dynamic Uncertainty utility function, enabling accurate utility estimation without extra inference. In addition, we developed a coarse-to-fine selection approach guided by a multi-armed bandit mechanism. Experiments show \model achieves 6.1\%-10.8\% performance improvement using only 2.5\% training data and reduces training costs by 5-10×.

\clearpage

\bibliographystyle{ACM-Reference-Format}
\bibliography{sample-base}

\end{document}